\documentclass[11pt]{article}
\usepackage[margin=1in]{geometry}
\usepackage{authblk} 
\usepackage[T1]{fontenc}
\usepackage{lmodern}

\usepackage{amsmath,amsthm,amssymb,amsfonts}
\usepackage{color}
\usepackage{graphicx}
\usepackage{subcaption}
\usepackage{url}

\usepackage{algorithm}
\usepackage{algorithmicx}
\usepackage{algpseudocode}

\usepackage{tikz}
\usepackage{pgfplots}
\usepackage{multirow}

\usepackage{stackengine,scalerel}
\newcommand\pland{\mathbin{\ensurestackMath{\stackinset{c}{}{c}{-1.6pt}%
		{\cdot}{\wedge}}}}

 \newtheorem{lemma}{Lemma}

\renewcommand{\phi}{\varphi}
\newcommand{\cland}[1]{\mathbin{\wedge_{#1}}}
\usepackage{todonotes}

\newcommand{\set}[1]{\left\{ #1 \right\}}

\newcommand{\Card}{\text{Card}}
\newcommand{\score}{\textit{Score}}
\newcommand{\len}{\textit{len}}

\newcommand{\LTLf}{\textbf{LTL\textsubscript{f}}}
\newcommand{\LTL}{\textbf{LTL}}

\newcommand{\dLTL}{\textbf{dLTL}}
\newcommand{\sym}{\textit{symbols}}
\newcommand{\prop}{\mathcal{P}}
\DeclareMathOperator{\lX}{\mathbf{X}}
\DeclareMathOperator{\lF}{\mathbf{F}}
\DeclareMathOperator{\lG}{\mathbf{G}}
\DeclareMathOperator{\lU}{\mathbf{U}}
\DeclareMathOperator{\limplies}{\rightarrow}
\DeclareMathOperator{\last}{\mathbf{last}}
\DeclareMathOperator{\false}{\mathit{false}}
\DeclareMathOperator{\true}{\mathit{true}}
\DeclareMathOperator{\sample}{\mathcal{S}}

\newcommand{\LastUsedPos}{\textsc{LastPos}}
\newcommand{\Index}{\textsc{Index}}

\newcommand{\tool}{\texttt{SCARLET}}
\newcommand{\flie}{\texttt{FLIE}}
\newcommand{\sys}{\texttt{SYSLITE\textsubscript{L}}}
\newcommand{\syslite}{\texttt{SYSLITE}}

\usepackage{complexity}
\usepackage{framed}

\usepackage[english]{babel}
\usepackage{amsmath,amsthm,amssymb,amsfonts}
\usepackage{color}
\usepackage{graphicx}
\usepackage{subcaption}
\usepackage{url}
\usepackage{hyperref}
\usepackage{authblk}

\usepackage{algorithm}
\usepackage{algorithmicx}
\usepackage{algpseudocode}

\usepackage{tikz}
\usepackage{pgfplots}
\usepackage{multirow}

\newtheorem{theorem}{Theorem}

\usepackage{graphicx}
\usepackage{enumitem}
\usepackage{multirow}%
\usepackage{amsmath,amssymb,amsfonts}%
\usepackage{amsthm}%
\usepackage{mathrsfs}%
\usepackage[title]{appendix}%
\usepackage{xcolor}%
\usepackage{textcomp}%
\usepackage{manyfoot}%
\usepackage{booktabs}%
\usepackage{algorithm}%
\usepackage{algorithmicx}%
\usepackage{algpseudocode}%
\usepackage{listings}
\usepackage{titlesec}
\usepackage{float}
\usepackage[numbers]{natbib}

\begin{document}

\title{A Scalable Anytime Algorithm for Learning Fragments of Linear Temporal Logic\footnote{A preliminary version of this article was published in the proceedings of the International Conference on Tools and Algorithms for the Construction and Analysis of Systems (TACAS) in 2022~\cite{Raha.Roy.ea:2022}. This full version includes new proofs of theoretical guarantees, extended experimental section, and improved presentation of the main algorithms.}}

\author[1]{Ritam Raha}

\author[2]{Rajarshi Roy}

\author[3]{Nathana{\"e}l Fijalkow}
\author[4,5]{Daniel Neider}

\affil[1]{Max-Planck Institute for Software Systems, Kaiserslautern, Germany}

\affil[2]{University of Oxford, Oxford, United Kingdom}

\affil[3]{CNRS, LaBRI, Universit{\'e} de Bordeaux, Bordeaux, France}
\affil[4]{TU Dortmund University, Dortmund, Germany}

\affil[5]{Center for Trustworthy Data Science and Security, University Alliance Ruhr, Dortmund, Germany}
\date{}

\maketitle

\begin{abstract}
Linear temporal logic ($\LTLf$) is a specification language for finite sequences (called traces) widely used in program verification, motion planning in robotics, process mining, and many other areas. We consider the problem of learning formulas in fragments of $\LTLf$ without the $\lU$-operator for classifying traces; despite a growing interest of the research community, existing solutions suffer from two limitations: they do not scale beyond small formulas, and they may exhaust computational resources without returning any result.
We introduce a new algorithm addressing both issues: our algorithm is able to construct formulas an order of magnitude larger than previous methods, and it is anytime, meaning that it in most cases successfully outputs a formula, albeit possibly not of minimal size.
We evaluate the performances of our algorithm using an open source implementation against publicly available benchmarks.
\end{abstract}

\section{Introduction}\label{sec:intro}

Linear Temporal Logic ($\LTLf$) is a prominent logic for specifying temporal properties~\cite{Pnueli77}. It has become a de facto standard in many fields such as model checking, program analysis, and motion planning for robotics. Linear Temporal Logic has been studied over finite traces~\cite{Vardi13} and called $\LTLf$. Over the past five to ten years learning temporal logics (of which $\LTLf$ is the core) has become an active research area and identified as an important goal in artificial intelligence: it formalises the difficult task of building explainable models from data.
Indeed, as we will see in the examples below and as argued in the literature~\cite{CamachoMcIlraith2019,RoyFismanNeider20,Raha23},
$\LTLf$ formulas are typically easy to interpret by human users and therefore useful as explanations.
The variable free syntax of $\LTLf$ and its natural inductive semantics make $\LTLf$ a natural target for building classifiers separating positive from negative traces.

The fundamental problem we study here is to build an explainable model in the form of an $\LTLf$ formula
from a set of positive and negative traces. More formally (we refer to the next section for formal definitions), given a set $u_1,\dots,u_n$ of positive traces and a set $v_1,\dots,v_m$ of negative traces, the goal is to construct a formula $\varphi$ of $\LTLf$ which satisfies all $u_i$'s and none of the $v_i$'s.
In that case, we say that $\varphi$ is a separating formula or---using machine learning terminology---a classifier.

To make things concrete let us introduce our running example, a classic motion planning problem in robotics and inspired by \cite{GroverBTK21}. A robot collects wastebin contents in an office-like environment and empties them in a trash container.
Let us assume that there is an office $\textrm{o}$, a hallway $\textrm{h}$, a container $\textrm{c}$ and a wet area $\textrm{w}$.
The following are possible traces obtained in experimentation with the robot (for instance, through simulation):
\[
\begin{array}{lll}
	u_1 & = & \textrm{h} \cdot \textrm{h} \cdot \textrm{h} \cdot \textrm{h} \cdot \textrm{o}
	\cdot \textrm{h} \cdot \textrm{c} \cdot \textrm{h}\\
	v_1 & = & \textrm{h} \cdot \textrm{h} \cdot \textrm{h} \cdot \textrm{h} \cdot \textrm{h} \cdot \textrm{c}
	\cdot \textrm{h} \cdot \textrm{o} \cdot \textrm{h} \cdot \textrm{h} \\
\end{array}
\]
In $\LTLf$ learning we start from these labelled data: given $u_1$ as positive and $v_1$ as negative, what is a possible classifier including $u_1$ but not $v_1$?
Informally, $v_1$ being negative implies that the order is fixed: $\textrm{o}$ must be visited before $\textrm{c}$.
We look for classifiers in the form of separating formulas, for instance
\[
\lF(\textrm{o} \wedge \lF \lX \textrm{c}),
\] 
where the $\lF$-operator stands for ``finally'' and $\lX$ for ``next''.
Note that this formula requires to visit the office first and only then visit the container.

Assume now that two more negative traces were added:
\[
\begin{array}{lll}
	v_2 & = & \textrm{h} \cdot \textrm{h} \cdot \textrm{h} \cdot \textrm{h} \cdot \textrm{h} \cdot \textrm{o} \cdot \textrm{w} \cdot \textrm{c} \cdot \textrm{h} \cdot \textrm{h} \cdot \textrm{h}\\
	v_3 & = & \textrm{h} \cdot \textrm{h} \cdot \textrm{h} \cdot \textrm{h} \cdot \textrm{h} \cdot \textrm{w} \cdot \textrm{o} \cdot \textrm{w} \cdot \textrm{c} \cdot \textrm{w} \cdot \textrm{w}
\end{array}
\]
Then the previous separating formula is no longer correct, and a possible separating formula is
\[
\lF(\textrm{o} \wedge \lF \lX \textrm{c}) \wedge \lG (\neg\textrm{w}),
\]
which additionally requires the robot to never visit the wet area.
Here the $\lG$-operator stands for ``globally''.

Let us emphasise at this point that for the sake of presentation, we consider only exact classifiers:
a separating formula must satisfy all positive traces and none of the negative traces.
However, as we will show our algorithm naturally extends to the noisy data setting where the goal is to construct an approximate classifier,
replacing `all' and `none' by `almost all' and `almost none'.

\paragraph*{Motivations and applications.}
The study of $\LTLf$ learning spans various research fields, each with unique applications, methodologies, and perspectives. This paragraph offers a brief overview of the motivations and uses of $\LTLf$ learning across these different communities.
\begin{itemize}[leftmargin=*, itemsep=1ex]	

	\item \textbf{Software Engineering, Programming Languages, and Formal Methods.}
Learning Linear Temporal Logic over finite traces ($\LTLf$ learning) is a specific type of specification mining, an active research area focused on automatically extracting formal specifications from code.  It is crucial to differentiate between dynamic and static settings in this context. This paper focuses on the dynamic setting, where we analyze program executions (traces) to infer properties of the code.
The term ``specification mining'' was introduced by Ammons et al. \cite{Ammons.Bodk.ea:2002} in a foundational paper that used finite state machines to represent both temporal and data dependencies. Zeller \cite{Zeller:2010} provided a roadmap for specification mining, emphasizing its potential benefits for software engineering. Later, Rozier \cite{Rozier:2016} established a research program centered around ``$\LTLf$ Genesis,'' which motivated and introduced the $\LTLf$ learning problem as we understand it today.
Texada, the first tool to support all $\LTLf$ formulas, was developed in 2015 by Lemieux et al. \cite{LPB15,Lemieux.Beschastnikh:2015}. Subsequent research has focused on scaling $\LTLf$ learning to handle industrial-sized datasets, leading to various approaches. 
Recently, Valizadeh et al. \cite{Valizadeh.Fijalkow.ea:2024} proposed a GPU-accelerated algorithm for $\LTLf$ learning.
Applications of specification mining in software engineering include detecting malicious behaviors \cite{Christodorescu.Jha.ea:2007} and violations \cite{Li.Zhou:2005}. It is already widely used: for example, the ARSENAL and ARSENAL2 projects \cite{Ghosh.Elenius.ea:2016} successfully constructed $\LTLf$ formulas from traces derived from English requirements, and the FRET project generates $\LTLf$ from trace descriptions \cite{Giannakopoulou.Pressburger.ea:2020}.
For a comprehensive overview of specification mining for software and its connection to data mining, refer to the textbook by Lo et al. \cite{Lo.Khoo.ea:2017} and the PhD thesis of Li \cite{Li:2013}.

\item \textbf{Control of Cyber-Physical Systems and Robotics.}
Another significant area where $\LTLf$ learning plays a vital role is in understanding the behavior of trajectories in models and systems. However, in this domain, with its emphasis on quantitative analysis, Signal Temporal Logic (\textbf{STL}) often takes precedence.  Unlike $\LTLf$, \textbf{STL} incorporates numerical constants, making it better suited to capturing real-valued and time-varying behaviors.
Specifically, temporal logics like \textbf{STL} enable researchers to analyze system robustness \cite{Bartocci.Bortolussi.ea:2015} and detect anomalies \cite{Jones.Kong.ea:2014,Kong.Jones.ea:2017}.
A considerable body of work exists on learning temporal logics in control and robotics, broadly categorized into two approaches:
\begin{itemize}[leftmargin=1.5em, noitemsep]
	\item \emph{Parameter fitting:} These techniques focus on optimizing the parameters of a predefined \textbf{STL} formula \cite{Asarin.Donze.ea:2012,Yang.Hoxha.ea:2012,Jin.Donze.ea:2015}.
	\item \emph{Formula and parameter search:} These methods aim to learn both the structure of the formula and its parameters \cite{Jin.Donze.ea:2015,BoVaPeBe-HSCC-2016}.
\end{itemize}
This research has led to numerous case studies, including:
\begin{itemize}[leftmargin=1.5em, noitemsep]
	\item \emph{Automotive applications:} Analyzing automobile transmission controllers and engine airpath control \cite{Yang.Hoxha.ea:2012}.
	\item \emph{Healthcare:} Modeling assisted ventilation for intensive care patients \cite{Bufo.Bartocci.ea:2014}.
	\item \emph{Biological systems:} Studying the dynamics of biological circadian oscillators and identifying different types of cardiac malfunction from electrocardiogram data \cite{Bartocci.Bortolussi.ea:2014}.

	\item \emph{Maritime safety:} Detecting anomalies in maritime environments \cite{BoVaPeBe-HSCC-2016}.
	\item \emph{Robotics:} Demonstrating learned behaviors in robots \cite{ChouOB20}.
	\item \emph{Aviation safety:} Detecting attention loss in pilots \cite{Lyu.Li.ea:2024}.
\end{itemize}

\item \textbf{Artificial Intelligence.}
The field of AI has not only broadened the applications of $\LTLf$ learning but also significantly enriched its techniques.  A central idea in AI's approach to $\LTLf$ learning is that $\LTLf$ provides a natural and interpretable framework \cite{CamachoMcIlraith2019} for defining objectives \cite{Icarte.Klassen.ea:2018} in various machine learning scenarios. For example, in reinforcement learning, $\LTLf$ formulas have been used to guide \cite{Camacho.Chen.ea:2017,Brafman.Giacomo.ea:2018} or to constraint~\cite{Icarte.Klassen.ea:2018*1,Camacho.Icarte.ea:2019} the learning process of policies.  Another recent application leverages $\LTLf$ to accelerate synthesis by decoupling data and control \cite{Murphy.Holzer.ea:2024}.
\end{itemize}
\paragraph*{SOTA on $\LTLf$  Learning.}
Different approaches have been explored. 
A first family of approaches leverages logical reasoning, such as SAT solving and decision trees~\cite{NeiderGavran18}, constraint programming~\cite{ArifLERCT20}, evolutionary algorithms~\cite{Bufo.Bartocci.ea:2014}, inductive logic programming~\cite{Ielo.Law.ea:2023}, alternating automata~\cite{CamachoMcIlraith2019}, and automata learning~\cite{Jeppu.Melham.ea:2020, Jeppu.Melham.ea:2022, Jeppu.Melham.ea:2023}.
A GPU-accelerated algorithm was recently published~\cite{Valizadeh.Fijalkow.ea:2024}, yielding an improved state of the art in a different category.
The second family is based on machine learning algorithms. Some algorithms use Bayesian inference, for positive only traces~\cite{Shah.Kamath.ea:2018} and for both positive and negative traces~\cite{ijcai2019-0776}. Many others employ deep learning: for positive only traces~\cite{Peng.Liang.ea:2023}, in combination with optimization~\cite{Gupta.Komp.ea:2024}, using Graph Neural Networks~\cite{Luo.Liang.ea:2022}, or even Transformers and Mamba architectures~\cite{Isik.Gol.ea:2024}.

Existing methods do not scale beyond formulas of small size, making them hard to deploy for industrial cases. A second serious limitation is that they often exhaust computational resources without returning any result. Indeed theoretical studies~\cite{FijalkowLagarde21,Mascle.Fijalkow.ea:2023} have shown that constructing the minimal $\LTLf$ formula is NP-hard already for very small fragments of $\LTLf$, explaining the difficulties found in practice.

\paragraph*{Our approach.}  To address both issues, we turn to \textit{approximation} and \textit{anytime} algorithms.
Here \textit{approximation} means that the algorithm does not ensure minimality of the constructed formula: 
it does ensure that the output formula separates positive from negative traces, but it may not be the smallest one.
On the other hand, an algorithm solving an optimisation problem is called \textit{anytime} if it finds better and better solutions the longer it keeps running. 
In other words, anytime algorithms work by refining solutions.
As we will see in the experiments, this implies that even if our algorithm timeouts it may yield some good albeit non-optimal formula.
Our algorithm targets a strict fragment of $\LTLf$, which does not contain the Until operator (nor its dual Release operator).
It combines two ingredients:
\begin{itemize}
	\item \textit{Searching for directed formulas}: we define a space efficient dynamic programming algorithm for enumerating formulas from a fragment of $\LTLf$ that we call Directed $\LTLf$ . 
	\item \textit{Combining directed formulas}: we construct two algorithms for combining formulas using Boolean operators. The first is an off-the-shelf \textit{decision tree algorithm}, and the second is a new greedy algorithm called \textit{Boolean set cover}. 
\end{itemize}

The two ingredients yield two subprocedures: the first one finds directed formulas of increasing size,
which are then fed to the second procedure in charge of combining them into a separating formula.
This yields an anytime algorithm as both subprocedures can output separating formulas even with a low computational budget
and refine them over time.

\vskip1em
Let us illustrate the two subprocedures in our running example.
The first procedure enumerates so-called \textit{directed formulas} in increasing size;
we refer to the corresponding section for a formal definition.
The directed formulas $\lF(\textrm{o} \wedge \lF \lX \textrm{c})$ and $\lG (\neg \textrm{w})$ have small size
hence will be generated early on.
The second procedure constructs formulas as Boolean combinations of directed formulas.
Without getting into the details of the algorithms, let us note that 
both $\lF(\textrm{o} \wedge \lF \lX \textrm{c})$ and $\lG (\neg \textrm{w})$ satisfy $u_1$.
The first does not satisfy $v_1$ and the second does not satisfy $v_2$ and $v_3$.
Hence, their conjunction $\lF(\textrm{o} \wedge \lF \lX \textrm{c}) \wedge \lG  (\neg \textrm{w})$ is separating, 
meaning it satisfies $u_1$ but none of $v_1,v_2,v_3$.

\paragraph*{Outline.}
The mandatory definitions and the problem statement we deal with are described in Section~\ref{sec:preliminaries}.
Section~\ref{sec:highlevel} shows a high-level overview of our main idea in the algorithm.
The next two sections, Section~\ref{sec:directed_formulas} and Section~\ref{sec:combining} describe the two phases of our algorithm in detail, in one section each. 
We discuss the theoretical guarantees and extensions of our algorithm in Section~\ref{sec:algorithm}.
We conclude with an empirical evaluation in Section~\ref{sec:experiments}.

\section{Preliminaries}\label{sec:preliminaries}

\paragraph*{Traces.} Let $\prop$ be a finite set of atomic propositions. 
An \textit{alphabet} is a finite non-empty set $\Sigma= 2^ \prop$, whose elements are called \sym.
A finite \emph{trace} over $\Sigma$ is a finite sequence 
$t = a_1 a_2 \ldots a_n$ such that for every $1 \leq i \leq n$, $a_i \in \Sigma$. 
We say that $t$ has length $n$ and write $\len(t) = n$.
For example, let $\prop = \{p,q\}$, in the trace 
$t = \{p,q\} \cdot \{p\} \cdot \{q\}$ both $p$ and $q$ hold at the first position, only $p$ holds in the second position, and $q$ in the third position. 
Note that, throughout the paper, we only consider finite traces.

A trace is a \textit{word} if exactly one atomic proposition holds at each position: we used words in the introduction example for simplicity, writing $h \cdot o \cdot c$ instead of $\set{h} \cdot \set{o} \cdot \set{c}$.

Given a trace $t = a_1 a_2 \ldots a_n$ and $1 \le i \leq j \leq n$, let $t[i,j] = a_i \ldots a_j$ be the \textit{infix} of $t$ from position $i$ up to and including position $j$. 
Moreover, $t[i] = a_i$ is the symbol at the $i^{th}$ position.

\paragraph*{Linear Temporal Logic.} In this paper, we are interested in $\LTLf$, defined by the following grammar

\[
\varphi:= p \in \prop \mid \neg p \mid \varphi \lor \psi \mid \varphi \land \psi \mid \lX \varphi \mid \lF \varphi \mid \lG \varphi
\mid \varphi \lU \psi
\]

We use the standard formulas:
$\true = p \lor \neg p$,  $\false = p \land \neg p$ and $\last = \neg \lX \true$, which denotes the last position of the trace.
As a shorthand, we use $\lX^n \varphi$ for $\underbrace{\lX \dots \lX}_{n \text{ times}} \varphi$.

The \emph{size $|\varphi|$ of a formula $\varphi$} is the size of its underlying syntax tree.

Formulas in \LTLf{} are evaluated over finite traces. 
To define the semantics of \LTLf{} we introduce the notation $t, i \models \varphi$,
which reads `the \LTLf{} formula $\varphi$ holds over trace $t$ from position $i$'.
We say that $t$ satisfies $\varphi$ and we write $t \models \varphi$ when $t,1 \models \varphi$.
The definition of $\models$ is inductive on the formula $\varphi$:

\begin{itemize}
	\item[--]  $t,i \models p \in \prop$ if $p \in t[i]$.
	\item[--] $t,i \models \lX \varphi$ if $i < \len(t)$ and $t, i+1 \models \varphi$. It is called the ne$\lX$t operator.
	\item[--] $t,i \models \lF \varphi$ if $t,i'\models \varphi$ for some $i'\in [i,\len(t)]$. It is called the eventually operator (F comes from $\lF$inally).
	\item[--] $t,i \models \lG \varphi$ if $t,i'\models \varphi$ for all $i' \in [i,\len(t)]$. It is called the $\lG$lobally operator.
	\item[--] $t,i \models \varphi \lU \psi$ if $t,j \models \psi$ for some $i \leq j \leq \len(t)$  and $t,i' \models \varphi$ for all $i \leq i' < j$. It is called the $\lU$ntil operator.
\end{itemize}

In this paper we do not consider the $\lU$ operator. Thus, for readability, we use $\LTLf$ to refer to the fragment $\LTLf \setminus \lU$ in the rest of the paper.

%
%

\vskip1em
The $\LTLf$ exact learning problem we study here is in the \emph{passive learning} setting: models are learnt based on a given data set. First, we define the input of our problem formulation, which we call a \emph{Sample}. 

\paragraph*{Sample.} A \emph{sample} consists of a set of labelled (finite) traces. Formally, we rely on a sample $\sample = (P, N)$ consisting of $P$, a set of positive traces, and $N$, a set of negative traces. We say a sample is \emph{informative} if $P \cap N = \emptyset$. We say an $\LTLf$ formula $\varphi$ is separating for $\sample$ if it satisfies all the positive traces in $P$ and does not satisfy any negative traces in $N$.

There are two relevant parameters for a sample: its \textit{size} $|\sample| = |P| + |N|$, which is the number of traces,
and its \textit{length} $\mathit{len}(\sample) = \max\{\len(t) \mid t \in P \cup N\}$, which is the maximum length of all traces.

\paragraph*{$\LTLf$ exact learning.} Given a sample $\sample = P \cup N$, construct a minimal $\LTLf$ formula $\varphi$ that is separating for $\sample$ i.e., $t \models \varphi$ for all $t \in P$ and $t \not \models \varphi$ for all $t \in N$.

\vskip1em
In practical scenarios, noise in data is omnipresent, i.e., they are misclassified as positive or negative behaviours of the underlying system. Exact learning algorithms on these noisy samples would often result in overfitting on the particular input data, and the extracted $\LTLf$ specifications might not be suitable to interpret the system properly. Hence, the problem is naturally extended to the $\LTLf$ noisy learning problem where the goal is to infer a separating $\LTLf$ formula with \emph{low loss}, where loss indicates the fraction of the sample misclassified by the formula. Given a sample $\sample = (P,N)$ and a formula $\varphi$, let us define the loss function $\mathit{loss}(\sample, \varphi) =  \frac{\mathrm{mc}(\sample, \varphi)}{|\sample|}$ where, $\mathrm{mc}(\sample, \varphi) = |\{t \not \models \varphi \mid t \in P\}| + |\{t  \models \varphi \mid t \in N\}|$ denotes the number of traces that $\varphi$ misclassifies. Given a threshold $\varepsilon$, an $\LTLf$ formula $\varphi$ is a  $\varepsilon$-separating formula if $\mathit{loss}(\sample, \varphi) \leq \epsilon$. Based on this definition of the loss function, we define our $\LTLf$ noisy learning problem as follows:

\paragraph*{$\LTLf$ noisy learning.} Given a sample $\sample$ and a threshold $\epsilon$, construct a minimal $\LTLf$ separating formula $\varphi$ for $\sample$ such that $\mathit{loss}(\sample, \varphi) \leq \epsilon$.

\vskip1em
Next, we present an algorithm for solving the $\LTLf$ exact learning problem; we later show how to adapt it to the noisy setting.

\section{High-level view of the approach}\label{sec:highlevel}
Let us start with a naive algorithm for the $\LTLf$ Learning problem. 
We can search through all $\LTLf$ formulas in some order and check whether they are separating for our sample or not.
Checking whether an $\LTLf$ formula is separating can be done using standard methods (\textit{e.g.} using bit vector operations~\cite{BaresiKR15}). 
However, the major drawback of this idea is that we have to search through all $\LTLf$ formulas, which is hard as the number of $\LTLf$ formulas grows very quickly\footnote{The number of $\LTLf$ formulas of size $k$ is asymptotically equivalent to $\frac{\sqrt{14}\cdot 7^k}{2 \sqrt{\pi k^3}}$~\cite{Flajolet08analyticcombinatorics}}.

\paragraph*{Directed \LTLf.}
The first insight of our approach is the definition of a fragment of \LTLf{} that we call \emph{directed \LTLf{}}: instead of entire $\LTLf$, our algorithm performs an iterative search through directed $\LTLf$ formulas.
As we will demonstrate in Section~\ref{sec:directed_formulas}, the definition of directed $\LTLf$ enables a very efficient search procedure.

Let us first define directed \LTLf. A \textit{partial symbol} is a conjunction of positive or negative distinct atomic propositions (true and false are not partial symbols).
We write $s = p_0 \wedge p_2 \wedge \neg p_1$ for the partial symbol specifying that $p_0$ and $p_2$ hold and $p_1$ does not.
The definition of a symbol satisfying a partial symbol is natural: for instance the symbol $\set{p_0,p_2,p_4}$ satisfies $s$.
The \emph{width} of a partial symbol is the number of atomic propositions it uses.

Directed $\LTLf$  is defined by the following grammar:
\[
\varphi = 
\lX^n s 
\quad \mid \quad 
\lF \lX^n s 
\quad \mid \quad 
\lX^n (s \wedge \varphi)
\quad \mid \quad 
\lF \lX^n (s \wedge \varphi),
\]
where $s$ is a partial symbol and $n \in \mathbb{N}$.
As an example, the directed formula
\[
\lF ((p \wedge q) \wedge \lF \lX^2 \neg p)
\]
reads: there exists a position satisfying $p \wedge q$, and at least two positions later, there exists a position satisfying $\neg p$.
The intuition behind the term ``directed'' is that a directed formula fixes the order in which the partial symbols occur.
A non-directed formula is $\lF p \wedge \lF q$: it does not specify whether $p$ appears before or after $q$.
Note that Directed $\LTLf$  only uses the $\lX$ and $\lF$ operators as well as conjunctions and atomic propositions.


\paragraph*{Boolean combinations.}
To include more formulas in our search space, we generate and search through Boolean combinations of Directed \textbf{LTL} formulas (Line~\ref{line:bool-comb}), which we describe in detail in Section~\ref{sec:combining}.

\paragraph*{Theoretical properties.}
The following result shows the relevance of our approach using directed $\LTLf$ and Boolean combinations.

\vskip1em
\begin{theorem}\label{thm:fragment-equivalence}
	Every formula of $\LTLf\,(\lF,\lX,\wedge,\vee)$ is equivalent to a Boolean combination of directed formulas.
	Equivalently, every formula of $\LTLf\,(\lG,\lX,\wedge,\vee)$ is equivalent to a Boolean combination of dual directed formulas.
\end{theorem}
\vskip1em

To get an intuition, let us consider the formula $\lF p \wedge \lF q$, which is not directed,
but equivalent to $\lF (p \wedge \lF q) \vee \lF (q \wedge \lF p)$.
In the second formulation, there is a disjunction over the possible orderings of $p$ and $q$.
The formal proof generalises this rewriting idea. For readability, in this section, we refer to Directed \textbf{LTL} as $\dLTL$ and the Boolean combination of Directed \textbf{LTL} as $\dLTL_{\wedge,\vee}$. 

We first prove two lemmas necessary for the proof of the above theorem.

\vskip1em
\begin{lemma}\label{lem:X}
	Let $\varphi$ be a $\dLTL$ formula. Then $\lX \varphi$ is a $\dLTL$ formula.
\end{lemma}
\vskip1em

\begin{proof}[Proof of Lemma~\ref{lem:X}]
The proof is inductive, following the grammar defining $\dLTL$.
The two non-trivial cases are solved by noting that $\lX \lF \lX^n \varphi$ is equivalent to $\lF \lX^{n+1} \varphi$.
\end{proof}

\vskip1em
\begin{lemma}\label{lem:disjunction}
	Let $\Delta_1$, $\Delta_2$ be two $\dLTL$ formulas. Then, $\Delta_1\wedge\Delta_2$ can be written as a disjunction of formulas in $\dLTL$.
\end{lemma}
\vskip1em

\begin{proof}[Proof of Lemma~\ref{lem:disjunction}]
	To prove the lemma, we use an induction over the structure of $\Delta_1\land\Delta_2$ to show that it can be written as a disjunction of $\dLTL$ formulas. 
	As induction hypothesis, we consider all formulas $\Delta_1' \wedge \Delta_2'$, where at least one of $\Delta_1'$ and $\Delta_2'$ is a subformula of $\Delta_1$ and $\Delta_2$ respectively, can be written as a disjunction of $\dLTL$ formulas.
	
	The base case of the induction is when at least one of $\Delta_1$ and $\Delta_2$ is a partial symbol.
	In this case, $\Delta_1\wedge\Delta_2$ is itself a $\dLTL$ formula by definition of $\dLTL$ formulas.
	
	The induction step proceeds via case analysis on the possible root operators of the formulas $\Delta_1$ and $\Delta_2$
	\begin{itemize}
		
		\item Case: at least one of $\Delta_1$ and $\Delta_2$ is of the form $s \wedge \Delta$ for some partial symbol $s$.
		Without loss of generality, let us say $\Delta_1 = s \wedge \Delta$.
		In this case, $\Delta_1\wedge\Delta_2 = (s\wedge\Delta)\wedge\Delta_2 = s\wedge(\Delta\wedge\Delta_2)$. 
		By hypothesis, $\Delta\wedge\Delta_2 = \bigvee_i\Gamma_i$ for some $\Gamma_i$ in $\dLTL$. 
		Thus, $\Delta_1\wedge\Delta_2 = s\wedge \bigvee_i\Gamma_i = \bigvee_i(s\wedge\Gamma_i)$, which is a disjunction of $\dLTL$ formulas. 
		
		\item Case: $\Delta_1$ is of the form $\lX\delta_1$ and $\Delta_2$ is of the form $\lX\delta_2$ where $\delta_1$ and $\delta_2$ are $\dLTL$ formulas. 
		In this case, $\Delta_1\wedge\Delta_2 = \lX(\delta_1\wedge\delta_2)$. 
		By hypothesis, $\delta_1\wedge\delta_2 = \bigvee_i\gamma_i$ for some $\gamma_i$'s in $\dLTL$. 
		Thus, $\Delta_1\wedge\Delta_2 = \lX(\bigvee_i\gamma_i) = \bigvee_i\lX\gamma_i$, which is a disjunction of $\dLTL$ formulas. 
		
		\item Case: $\Delta_1$ is of the form $\lX\delta_1$ and $\Delta_2$ is of the form $\lF\delta_2$ where $\delta_1$ and $\delta_2$ are $\dLTL$ formulas.
		In this case, $\Delta_1\wedge\Delta_2 = \lX\delta_1\wedge\lF\delta_2 = (\lX\delta1 \wedge \delta_2) \vee (\lX\delta_1 \wedge \lF\lX\delta_2) = (\lX\delta1 \wedge \delta_2) \vee \lX(\delta_1 \wedge \lF\delta_2)$.
		By hypothesis, both formulas $(\lX\delta1 \wedge \delta_2)$ and $(\delta_1 \wedge \lF\delta_2)$ can be written as a disjunction of $\dLTL$ formulas.
		Thus, $\Delta_1\wedge\Delta_2$ can also be written as a disjunction of $\dLTL$ formulas
		
		\item Case: $\Delta_1$ is of the form $\lF\delta_1$ and $\Delta_2$ is of the form $\lF\delta_2$ where $\delta_1$ and $\delta_2$ are $\dLTL$ formulas.
		In this case, $\Delta_1\wedge\Delta_2 = \lF\delta_1 \wedge \lF\delta_2 = \lF(\delta_1\wedge \lF\delta_2)\vee\lF(\delta_2\wedge \lF\delta_1)$.
		By hypothesis, both formulas $\delta1 \wedge \lF\delta_2$ and $\delta_2 \wedge \lF\delta_1$ can be written as a disjunction of $\dLTL$ formulas.
		Thus, $\Delta_1\wedge\Delta_2$ can also be written as a disjunction of $\dLTL$ formulas.
	\end{itemize} 
\end{proof}

\begin{proof}[Proof of Theorem~\ref{thm:fragment-equivalence}]
	The proof proceeds via induction on the structure of formulas $\varphi$ in $\LTLf\,(\lF,\lX,\wedge,\vee)$.
	As induction hypothesis, we consider that all formulas $\varphi'$ which are structurally smaller than $\varphi$ can be expressed in $\dLTL_{\wedge,\vee}$.
	
	As the base case of the induction, we observe that formulas $p$ for all $p\in\mathcal{P}$, are $\dLTL$ formulas and thus, in $\dLTL_{\wedge,\vee}$.
	
	For the induction step, we perform a case analysis based on the root operator of $\varphi$.
	\begin{itemize}
		\item Case $\varphi=\varphi_1\vee\varphi_2$ or $\varphi=\varphi_1\wedge\varphi_2$: By hypothesis, $\varphi_1$ is in $\dLTL_{\wedge,\vee}$ and $\varphi_2$ is in $\dLTL_{\wedge, \vee}$. Now, $\varphi$ is in $\dLTL_{\wedge,\vee}$ since $\dLTL_{\wedge,\vee}$ is closed under positive boolean combinations.
		
		\item Case $\varphi=\lX\varphi_1$: By hypothesis, $\varphi_1 \in \dLTL_{\wedge,\vee}$ and thus $\varphi_1= \bigvee_j (\bigwedge_i \Delta_i)$. Now, $\varphi = \lX(\bigvee_j (\bigwedge_i \Delta_i)) = \bigvee_j (\lX(\bigwedge_i\Delta_i)) = \bigvee_j (\bigwedge_i\lX\Delta_i) = \bigvee_i\bigwedge_i \Delta'_i$ ($\lX\Delta_i$ is a $\dLTL$ formula). Thus, $\varphi$ is in $\dLTL_{\wedge,\vee}$.
		
		\item Case $\varphi=\lF\varphi_1$: By hypothesis, $\varphi_1 \in \dLTL_{\wedge,\vee}$ and thus 
		$\varphi_1 = \bigvee_j(\bigwedge_i\Delta_i)$. Now $\varphi = \lF\varphi_1 = \bigvee_j(\lF(\bigwedge_i\Delta_i))$. 
		Using lemma~\ref{lem:disjunction}, we can re-write $\bigwedge_i\Delta_i$ as $\bigvee_i\Gamma_i$ for some $\Gamma_i$'s in $\dLTL$. 
		As a result, $\varphi = \bigvee_j\bigvee_i\lF(\Gamma_i)$.
		Thus, $\varphi$ is in $\dLTL_{\wedge,\vee}$.
	\end{itemize}
\end{proof}

\section{Searching for directed formulas}\label{sec:directed_formulas}

Let us consider the following problem: given the sample $S = P \cup N$,
we want to generate all directed formulas together with the list of traces in $S$ they satisfy.
Our first technical contribution and key to the scalability of our approach is an efficient solution to this problem
based on dynamic programming.

Let us define a natural order in which we want to generate directed formulas. 
They have two parameters: \emph{length}, which is the number of partial symbols in the directed formula, and \emph{width}, which is the maximum of the widths of the partial symbols in the directed formula (recall that the width of a partial symbol is the number of atomic propositions it uses).
For instance, the directed formulas of length $1$ and width $1$ are exactly $\lX^k p$ and $\lF \lX^k p$ for $p$ an atomic proposition and $k$ bounded by the maximum length in the sample.

We consider the order based on summing these two parameters:
\[
(1,1), (2,1), (1,2), (3,1), (2,2), (1,3), \dots
\]
(We note that in practice, slightly more complicated orders on pairs are useful since we want to increase the length more often than the width.)
Our enumeration algorithm works by generating all directed formulas of a given pair of parameters in a recursive fashion. 
Assuming that we already generated all directed formulas for the pair of parameters $(\ell,w)$, we define two procedures,
one for generating the directed formulas for the parameters $(\ell+1, w)$, and the other one for $(\ell, w+1)$.

When we generate the directed formulas, we also keep track of which traces in the sample they satisfy by exploiting a dynamic programming table called 
$\LastUsedPos$.
We define it is as follows, where $\varphi$ is a directed formula
and $t$ a trace in $S$:
\[
\LastUsedPos(\varphi, t) = 
\set{i \in [1,\len(t)] : t[1,i] \models \varphi}.
\]
The main benefit of $\LastUsedPos$ is that it meshes well with directed formulas: it is algorithmically easy to compute them recursively on the structure of directed formulas.

A useful idea is to change the representation of the set of traces $S$,
by precomputing the lookup table $\Index$ defined as follows,
where $t$ is a trace in $S$, $s$ a partial symbol, and $i$ in $[1,\len(t)]$:
\[
\Index(t, s, i) = 
\set{j \in [i+1,\len(t)] : t[j] \models s}.
\]
The table $\Index$ can be precomputed in linear time from $S$,
and makes the dynamic programming algorithm easier to formulate.

Having defined the important ingredients, we now present the pseudocode~\ref{alg:search_directed} for both increasing the length and width of a formula. 

For the length increase algorithm, we define two extension operators $\cland{=k}$ and $\cland{\geq k}$ that ``extend'' the length of a directed formula $\varphi$ by including a partial symbol $s$ in the formula. 
Formally: let $s'$ denote the rightmost partial symbol in $\phi$, to obtain $s \cland{=k} \varphi$ we replace $s'$ with $(s' \land \lX^k s)$. 
Similarly, $s \cland{\geq k} \varphi$ replaces $s'$ with $(s' \land \lF \lX^k s)$. 
For instance, $c\cland{=2}\lX(a\land \lX b)= \lX(a\land \lX (b\wedge \lX^2c))$.
A naive implementation of the length increase algorithm would be to construct $s \cland{=k} \varphi$ and $s \cland{\geq k} \varphi$ for all formulas $\varphi$ and partial symbol $s$. Our algorithm relies on the data structure $\LastUsedPos$ and $\Index$ to construct only a subset of these, removing the formulas that do not yield new useful semantics.
More precisely, we only generate $s \cland{=k} \varphi$ for $k = j-i$ with $i \in \LastUsedPos(\varphi,t)$ and $j \in \Index(t,s,i)$,
and $s \cland{\geq k} \varphi$ for $k = j'-i$ with $i \in \LastUsedPos(\varphi,t)$ and $j' \leq \max \Index(t,s,i)$.
At the same time as generating, we compute $\LastUsedPos$ for the new formulas.

For the width increase algorithm, we say that two directed formulas are \emph{compatible} if they are equal except for partial symbols. 
For two compatible formulas, we define a \emph{pointwise-and} ($\pland$) operator that takes the conjunction of the corresponding partial symbols at the same positions. For instance, $\lX(a \land \lX b) \pland \lX(b \land \lX c) = \lX((a \land b) \land \lX(b \land c))$.

The actual implementation of the algorithm refines the algorithms in certain places. For instance:
\begin{itemize}
	\item Line~\ref{line:partial-symbol}: instead of considering all partial symbols, we restrict to those appearing in at least one positive trace.
	\item Line~\ref{line:f-optimization}: some computations for $\varphi_{\ge j}$ can be made redundant; a finer data structure factorises the computations.
	\item Line~\ref{line:compatible}: using a refined data structure, we only enumerate compatible directed formulas.
\end{itemize}

\begin{lemma}
	The length increase algorithm generates all directed formulas of parameters $(\ell + 1,w)$,
	and the width increase algorithm generates all directed formulas of parameters $(\ell,w + 1)$,
	and both correctly compute the tables \LastUsedPos. 
	Consequently, Algorithm~\ref{alg:search_directed} generates all directed formulas.
\end{lemma}

\paragraph*{The dual point of view.} We use the same algorithm to produce formulas in a dual fragment to Directed $\LTL$, which uses the $\lX$ and $\lG$ operators, the $\last$ predicate, as well as disjunctions and atomic propositions.
The only difference is that we swap positive and negative traces in the sample.
We obtain a directed formula from such a sample and apply its negation as shown below:
\[
\neg \lX \varphi = \last \vee \lX \neg \varphi \qquad ; \qquad
\neg \lF \varphi = \lG \neg \varphi \qquad ; \qquad
\neg (\varphi_1 \wedge \varphi_2) = \neg \varphi_1 \vee \neg \varphi_2.
\]
\begin{algorithm}[!t]
	\caption{Generation of directed formulas for the set of traces $S$}
	\label{alg:search_directed}
	\begin{algorithmic}[1]
		\Procedure{Search directed formulas -- length increase}{$\ell,w$}
		\ForAll{directed formulas $\varphi$ of length $\ell$ and width $w$} 
		\ForAll{partial symbols $s$ of width at most $w$}\label{line:partial-symbol}
		\ForAll{$t \in S$}
		\State $I = \LastUsedPos(\varphi, t)$
		\ForAll{$i \in I$}
		\State $J = \Index(t, s, i)$
		\ForAll{$j \in J$}
		\State $\varphi_{= j} \leftarrow s \cland{= (j-i)} \varphi$
		\State add $j$ to $\LastUsedPos(\varphi_{= j}, t)$
		\EndFor
		\ForAll{$j' \le \max(J)$}
		\State $\varphi_{\ge j'} \leftarrow s \cland{\geq (j'-i)} \varphi$\label{line:f-optimization}
		\State add $J \cap [j',\len(t)]$ to $\LastUsedPos(\varphi_{\ge j'}, t)$
		\EndFor
		\EndFor
		\EndFor
		\EndFor
		\EndFor
		\EndProcedure
		\State
		\Procedure{Search directed formulas -- width increase}{$\ell,w$}
		\ForAll{directed formulas $\varphi$ of length $\ell$ and width $w$} 
		\ForAll{directed formulas $\varphi'$ of length $\ell$ and width $1$} 
		\If{$\varphi$ and $\varphi'$ are compatible}\label{line:compatible}
		\State $\varphi'' \leftarrow \varphi \pland \varphi'$
		\ForAll{$t \in S$}
		\State $\LastUsedPos(\varphi'', t) \leftarrow \LastUsedPos(\varphi, t) \cap \LastUsedPos(\varphi', t)$
		\EndFor
		\EndIf
		\EndFor
		\EndFor
		\EndProcedure
	\end{algorithmic}
\end{algorithm}

\section{Boolean combinations of formulas}\label{sec:combining}

As explained in the previous section, we can efficiently generate directed formulas and dual directed formulas.
Now we explain how to form a Boolean combination of these formulas in order to construct separating formulas, as illustrated in the introduction.

\begin{figure}[H]
	\centering
	\includegraphics[width=0.5\columnwidth]{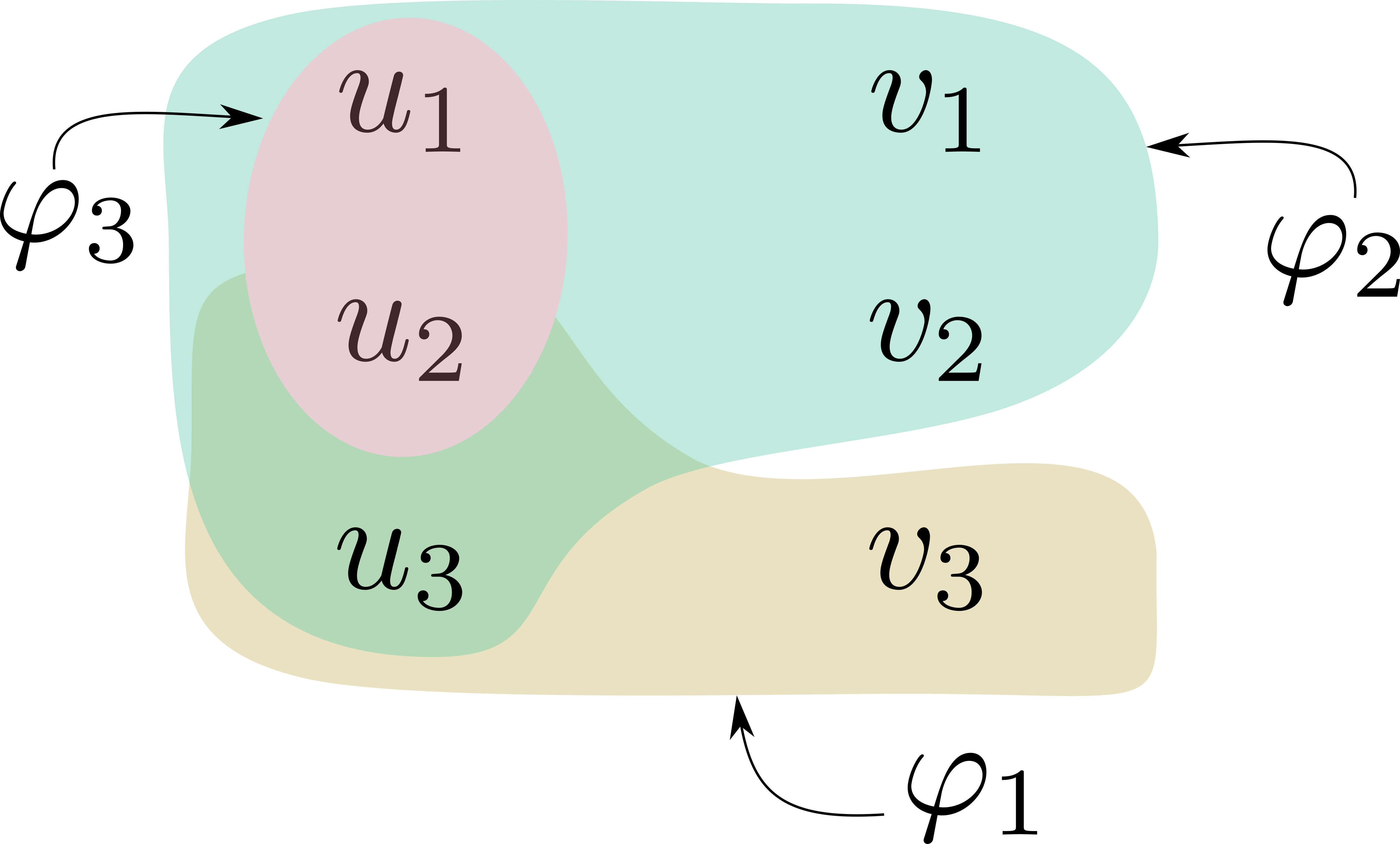}
	\caption{The Boolean set cover problem: the formulas $\varphi_1, \varphi_2$, and $\varphi_3$ satisfy the words encircled in the corresponding area; in this instance $(\varphi_1 \wedge \varphi_2) \vee \varphi_3$ is a separating formula.}
	\label{fig:boolean_set_cover}
\end{figure}

\paragraph*{Boolean combination of formulas.} Let us consider the following subproblem: given a set of formulas,
does there exist a Boolean combination of some of the formulas that is a separating formula?
We call this the \textbf{Boolean set cover} problem, which is illustrated in Figure~\ref{fig:boolean_set_cover}.
In this example we have three formulas $\varphi_1,\varphi_2$, and $\varphi_3$, each satisfying subsets of $u_1,u_2,u_3,v_1,v_2,v_3$ as represented in the drawing. Inspecting the three subsets reveals that $(\varphi_1 \wedge \varphi_2) \vee \varphi_3$ is a separating formula. 

The Boolean set cover problem is a generalization of the well known and extensively studied set cover problem,
where we are given $S_1,\dots,S_m$ subsets of $[1,n]$, and the goal is to find a subset $I$ of $[1,m]$ such that
$\bigcup_{i \in I} S_i$ covers all of $[1,n]$ -- such a set $I$ is called a cover.
Indeed, it corresponds to the case where all formulas satisfy none of the negative traces: in that case, conjunctions are not useful, and we can ignore the negative traces.
The set cover problem is known to be NP-complete. However, there exists a polynomial-time $\log(n)$-approximation algorithm called the greedy algorithm: it is guaranteed to construct a cover that is at most $\log(n)$ times larger than the minimum cover.
This approximation ratio is optimal in the following sense~\cite{DinurS14}: there is no polynomial time $(1 - o(1)) \log(n)$-approximation algorithm for set cover unless P = NP.
Informally, the greedy algorithm for the set cover problem does the following: 
it iteratively constructs a cover $I$ by sequentially adding the most `promising subset' to $I$,
which is the subset $S_i$ maximising how many more elements of $[1,n]$ are covered by adding $i$ to $I$.

\vskip1em
We introduce an extension of the greedy algorithm to the Boolean set cover problem.
The first ingredient is a scoring function, which takes into account both how close the formula is to being separating,
and how large it is. We use the following score:
\[
\score(\varphi) = \frac{\Card(\set{t\in P : t \models \varphi}) + \Card(\set{t\in N : t \not\models \varphi})}{\sqrt{|\varphi|} + 1},
\]
where $|\varphi|$ is the size of $\varphi$. 
The use of $\sqrt{\cdot}$ is empirical, it is used to mitigate the importance of size over being separating.

The pseudocode is given in Algorithm~\ref{alg:greedy}. The algorithm maintains a set of formulas $B$, which is initially the set of formulas given as input, 
and add new formulas to $B$ until finding a separating formula.
Let us fix a constant $K$, which in the implementation is set to $5$.
At each point in time, the algorithm chooses the $K$ formulas $\varphi_1,\dots,\varphi_K$ with the highest score in $B$ and constructs all disjunctions and conjunctions of $\varphi_i$ with formulas in $B$.
For each $i$, we keep the disjunction or conjunction with a maximal score, and add this formula to $B$ if it has higher score than $\varphi_i$.
We repeat this procedure until we find a separating formula or no formula is added to $B$. An important optimisation is to keep an upper bound on the size of a separating formula, which we use to cut off computations that cannot lead to smaller formulas in the greedy algorithm for the Boolean set cover problem.

\begin{algorithm}
	\caption{Greedy algorithm for the Boolean set cover problem}
	\label{alg:greedy}
	\begin{algorithmic}[1]
		\State \textbf{input:} $u_1,\dots,u_n,v_1,\dots,v_m$, and $B$ a set of formulas
		\State $K \leftarrow 5$
		
		\Procedure{Greedy}{$B$}
		\State choose the $K$ formulas $\varphi_1,\dots,\varphi_K$ in $B$ with highest score
		\ForAll{$\psi \in 
		B$}
		\ForAll{$i \in [1,K]$}
		\State construct $\psi \wedge \varphi_i$ and $\psi \vee \varphi_i$
		\State compute their scores
		\If{one of the two formulas is separating}
		\State \textbf{return} the separating formula
		\EndIf
		\EndFor
		\State let $\theta$ be the formula with highest score computed using $\psi$
		\If{$\theta$ has higher score than $\psi$}
		\State add $\theta$ to $B$
		\EndIf
		\EndFor
		\EndProcedure
	\end{algorithmic}
\end{algorithm}

\vskip1em
Another natural approach to the Boolean set cover problem is to use decision trees. 
The encoding is very natural: we use one variable for each trace and one atomic proposition for each formula to denote whether the trace satisfies the formula. We then construct a decision tree classifying all traces.

We experimented with both approaches and found that the greedy algorithm with a small enough choice for $K$ as a heuristic is both faster and yields smaller formulas. On the other hand, the formulas output using the decision tree approach are prohibitively larger and therefore not useful for explanations. Let us however remark that using decision trees we get a theoretical guarantee that the algorithm will always terminate with a separating formula (often consuming prohibitively large runtime), which might not always be the case for the greedy algorithm with a small value of $K$.

\section{The algorithm}\label{sec:algorithm}

\begin{algorithm} 
	\caption{Overview of our algorithm}
	\label{alg:main}
	\begin{algorithmic}[1]
		
		\State $B \gets \emptyset$
		\State $\psi \gets \text{none}$: best formula found
		\ForAll{$s$ in ``size order''}
		
		\State $D \gets$ all Directed $\LTLf$  formulas of parameter $s$ \label{line:dltl}
		\ForAll{$\varphi \in D$}
		\If{$\varphi$ is separating and smaller than $\psi$}\label{line:separating1}
		\State $\psi \leftarrow \varphi$\label{line:anytime1}
		\EndIf
		\EndFor
		\State $B \gets B \cup D$
		\State $B \gets$ Boolean combinations of the promising formulas in $B$\label{line:bool-comb}
		\ForAll{$\varphi \in B$}
		\If{$\varphi$ is separating and $|\varphi|<|\psi|$}\label{line:separating2}
		\State $\psi \leftarrow \varphi$\label{line:anytime2}
		\EndIf
		\EndFor
		\EndFor
		\State Return $\psi$
	\end{algorithmic}
\end{algorithm}

During the search for formulas, our algorithm searches for smaller separating formulas at each iteration than the previously found ones, if any.
Once a separating formula is found, we only search through formulas that are smaller than the found separating formula, which reduces the search space significantly.

\paragraph*{Anytime property.} The anytime property of our algorithm is also consequence of storing the smallest formula seen so far (Line~\ref{line:anytime1} and~\ref{line:anytime2}).
Once we find a separating formula, we can output it and continue the search for smaller separating formulas.

\paragraph*{Extension to the noisy setting.} The algorithm is seamlessly extended to the noisy setting by rewriting lines~\ref{line:separating1} and~\ref{line:separating2}:
instead of outputting only separating formulas, we output $\varepsilon$-separating formulas.

\vskip1em
Recall that in practice the Greedy algorithm for Boolean set cover fixes a bound $K$ and considers only the $K$ most promising formulas. On the theoretical front, little can be said about this algorithm for a small bound $K$. However, if we remove this heuristic, we obtain the following reassuring properties.

\begin{theorem}
The theoretical algorithm where $K$ is set to infinity satisfies the following properties: 
\begin{itemize}
	\item \textit{terminating}: given a bound on the size of formulas, the algorithm eventually generates all formulas of bounded size,
	\item \textit{correctness}: if the algorithm outputs a formula, then it is separating,
	\item \textit{completeness}: if there exists a separating formula in $\LTLf$ with no nesting of $\lF$ and $\lG$,
	then the algorithm finds a separating formula.
\end{itemize}
\end{theorem}

\section{Experimental evaluation}\label{sec:experiments}

\colorlet{flie}{red!50!white}
\colorlet{sys}{blue!50!white}
\colorlet{tool}{green!50!white}

In this section, we answer the following research questions to assess the performance of our $\LTLf$ learning algorithm.
\begin{description}
	\item[RQ1:] How effective are we in learning concise $\LTLf$ formulas from samples?
	\item[RQ2:] How much scalability do we achieve through our algorithm?
	\item[RQ3:] What do we gain from the anytime property of our algorithm?
	\item [RQ4:] How effective are we in the noisy learning setting?
\end{description}

\paragraph*{Experimental Setup.}
To answer the questions above, we have implemented a prototype of our algorithm in Python\,3 in a tool named \tool{}\footnote{\url{https://github.com/rajarshi008/Scarlet}} (SCalable Anytime algoRithm for LEarning lTl)~\cite{RahaRFNjoss24}.
We run \tool{} on several benchmarks generated synthetically from $\LTLf$ formulas used in practice.
To answer each research question precisely, we choose different sets of $\LTLf$ formulas. 
We discuss them in detail in the corresponding sections.
Note that, however, we did not consider any formulas with $\lU$-operator, since \tool{} is not designed to find such formulas.

To assess the performance of \tool{}, we compare it against two state-of-the-art tools for learning logic formulas from examples:
\begin{enumerate}
	\item \flie\footnote{\url{https://github.com/ivan-gavran/samples2LTL}}, developed by \cite{NeiderGavran18}, infers minimal $\LTLf$ formulas using a learning algorithm that is based on constraint solving (SAT solving).
	\item \syslite{}\footnote{\url{https://github.com/CLC-UIowa/SySLite}}, developed by \cite{ArifLERCT20}, originally infers minimal past-time $\LTLf$ formulas using an enumerative algorithm implemented in a tool called CVC4SY~\cite{ReynoldsBNBT19}.
	For our comparisons, we use a version of \syslite{} that we modified (which we refer to as \sys{}) to infer $\LTLf$ formulas rather than past-time $\LTLf$ formulas.
	Our modifications include changes to the syntactic constraints generated by \sys{} as well as changing the semantics from past-time $\LTLf$ to ordinary $\LTL$.
\end{enumerate}
To obtain a fair comparison against \tool{}, in both the tools, we disabled the $\lU$-operator.
This is because if we allow $\lU$-operator this will only make the tools slower since they will have to search through all formulas containing $\lU$.

All the experiments are conducted on a single core of a Debian machine with Intel Xeon E7-8857 CPU (at 3\,GHz) using up to 6\,GB of RAM. We set the timeout to be 900\,s for all experiments. We include scripts to reproduce all experimental results in a publicly available artifact~\cite{artifact}.

\vskip 1cm
\begin{table}[th]
	\vspace{-0.5cm}
	\centering
	\caption{Common \LTL{} formulas used in practice}
	\label{tab:LTL-patterns}
		\begin{tabular}{|rl|}
			\hline
			\rule{0pt}{1.2\normalbaselineskip} Absence: & $\lG( \lnot p)$,~$\lG( q \limplies \lG( \lnot p ) )$ \\[2mm]
			Existence: & $\lF( p )$,~$\lG( \lnot p) \lor \lF( p \land \lF( q ) )$ \\[2mm]
			Universality: & $\lG( p )$,~$\lG( q \limplies \lG( p ) )$ \\[2mm]
			\multirow{3}{*}{$\substack{$\text{Disjunction of}$\\ $\text{patterns:}$}$} & $\lG(\lnot p) \lor \lF(p \land \lF(q)$\\ 
			& \hspace{7mm}$\lor \lG(\lnot s) \lor \lF(r \land \lF(s))$,\\
			& $\lF(r) \lor \lF(p) \lor \lF(q)$\\
			\hline
		\end{tabular}
		\vspace{.5cm}
		
	\end{table}

	\paragraph*{Sample generation.} To provide a comparison among the learning tools, we follow the literature~\cite{NeiderGavran18,RoyFismanNeider20} and use synthetic benchmarks generated from real-world $\LTLf$ formulas.
	For benchmark generation, earlier works rely on a fairly naive generation method.
	In this method, starting from a formula $\varphi$, a sample is generated by randomly drawing traces and categorizing them into positive and negative examples depending on the satisfaction with respect to $\varphi$.
	This method, however, often results in samples that can be separated by formulas much smaller than $\varphi$.
	Moreover, it often requires a prohibitively large amount of time to generate samples (for instance, for $\lG p$, where almost all traces satisfy a formula) and, hence, often does not terminate in a reasonable time.
	
	To alleviate the issues in the existing method, we have designed a novel generation method for the quick generation of large samples.
	The outline of the generation algorithm is presented in Algorithm~\ref{alg:sample-generation}.
	The crux of the algorithm is to convert the $\LTLf$ formula $\varphi$ into its equivalent DFA $\mathcal{A}_\varphi$ and then extract random traces from the DFA to obtain a sample of desired length and size. 
	
	To convert $\varphi$ into its equivalent DFA $\mathcal{A}_{\varphi}$ (Line~\ref{line:ltl2dfa}), we rely on a python tool LTL\textsubscript{f}2DFA\footnote{https://github.com/whitemech/ltlf2DFA}. Essentially, this tool converts $\varphi$ into its equivalent formula in First-order Logic and then obtains a minimal DFA from the formula using a tool named MONA~\cite{KlaEtAl:Mona}.
	
	For extracting random traces from the DFA (Line~\ref{line:genrandomword1} and~\ref{line:genrandomword2}), we use a procedure suggested by \cite{Bernardi2010ALA}.
	The procedure involves generating words by choosing letters that have a higher probability of leading to an accepting state. 
	This requires assigning appropriate probabilities to the transitions of the DFA.  
	In this step, we add our modifications to the procedure. 
	The main idea is that we adjust the probabilities of the transitions appropriately to ensure that we obtain \emph{distinct} words in each iteration.

	\algdef{SE}{Loop}{EndLoop}[1]{\textbf{Loop} \(\mbox{#1}\)}{\textbf{end}}
	\renewcommand{\algorithmicrequire}{\textbf{Input:}}
	\begin{algorithm}
		\caption{Sample generation algorithm}
		\label{alg:sample-generation}
		\begin{algorithmic}[1]
			\Require Formula $\varphi$, length $\ell$, number of positive traces $n_P$, number of negative traces $n_N$.\\
			\State $P\gets \{\}$, $N\gets \{\}$ 
			\State $\mathcal{A}_\varphi$ $\gets$ convert2DFA($\varphi$)\label{line:ltl2dfa}
			\Loop { $n_P$ times}
			\State $w$ $\gets$ random accepted word of length $\ell$ from $\mathcal{A}_\varphi$.\label{line:genrandomword1}
			\State $P\gets P \cup \{w\}$
			\EndLoop
			\Loop {$n_N$ times } 
			\State $w$ $\gets$ random accepted word of length $\ell$ from $\mathcal{A}^{\textsf{c}}_\varphi$.\label{line:genrandomword2}
			\State $N \gets N \cup \{w\}$
			\EndLoop
			\State \textbf{return} $S=(P,N)$
		\end{algorithmic}
	\end{algorithm}
	
	Unlike existing sample generation methods, our method does not create random traces and try to classify them as positive or negative. This results in a much faster generation of samples.
	

	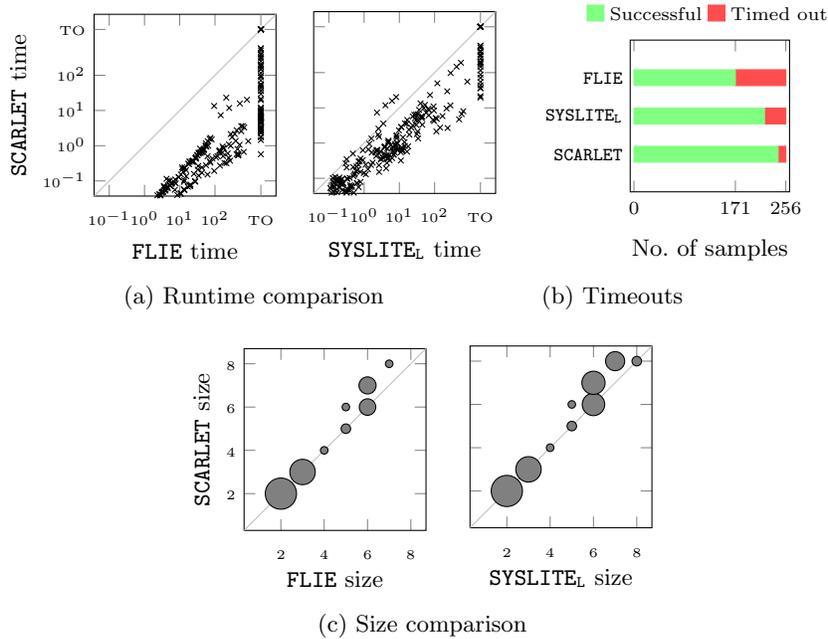
\begin{figure*}[t]
		\begin{center}
			\scalebox{1.0}{
				\hspace{-2.3cm}
				\subcaptionbox{Runtime comparison\label{subfig:runtime-comp}}{
					\begin{subfigure}[b]{0.25\textwidth}
						\centering
						\begin{tikzpicture}
							\begin{loglogaxis}[
								height=45mm,
								width=45mm,
								xmin=1e-1, ymin=1e-1,
								enlarge x limits=true, enlarge y limits=true,
								xlabel = {\flie{} time},
								ylabel = {\tool{} time},
								xtickten={-1, 0, 1, 2},%
								ytickten={-1, 0, 1, 2},%
								extra x ticks={2000}, extra x tick labels={\strut TO},%
								extra y ticks={2000}, 
								extra y tick labels={\strut TO},
								xminorticks=false,
								yminorticks=false,
								xlabel near ticks,
								ylabel near ticks,
								label style={font=\footnotesize},
								x tick label style={font={\strut\tiny}},
								y tick label style={font={\strut\tiny}},		
								x label style = {yshift=1mm},
								]
								
								\addplot[
								scatter=false,
								only marks,
								mark=x,
								mark size=1.5,
								]
								table [col sep=comma,x={ Flie Time},y={Collab Time}] {data/RQ1-tool-comparison.csv};
								
								\draw[black!25] (rel axis cs:0, 0) -- (rel axis cs:1, 1);
								
							\end{loglogaxis}
						\end{tikzpicture}
					\end{subfigure}
					\hskip 6mm
					\begin{subfigure}[b]{0.25\textwidth}
						\centering
						\begin{tikzpicture}
							\begin{loglogaxis}[
								height=45mm,
								width=45mm,
								xmin=1e-1, ymin=1e-1,
								enlarge x limits=true, enlarge y limits=true,
								xlabel = {\sys{} time},
								ylabel = {},
								xtickten={-1, 0, 1, 2},%
								ytickten={-1, 0, 1, 2},
								extra x ticks={2000}, extra x tick labels={\strut TO},%
								extra y ticks={2000}, 
								extra y tick labels={},
								xminorticks=false,
								yminorticks=false,
								xlabel near ticks,
								ylabel near ticks,
								label style={font=\footnotesize},
								x tick label style={font={\strut\tiny}},		
								x label style = {yshift=1mm},
								yticklabels = {},
								]
								\addplot[
								scatter=false,
								only marks,
								mark=x,
								mark size=1.5,
								]
								table [col sep=comma,x={Syslite Time},y={Collab Time}] {data/RQ1-tool-comparison.csv};
								
								\draw[black!25] (rel axis cs:0, 0) -- (rel axis cs:1, 1);
								
							\end{loglogaxis}
						\end{tikzpicture}
					\end{subfigure}
				}
				
				\subcaptionbox{Timeouts\label{subfig:timeouts}}{
					\begin{subfigure}[b]{0.17\textwidth}
						\centering
						\begin{tikzpicture}
							\scriptsize
							\begin{axis}[
								height=4cm,width=4.2cm,
								xbar stacked,
								bar width=6pt,
								enlarge x limits= {abs=0.5mm},
								enlarge y limits= {abs=5mm},
								xlabel= {Number of samples},
								xtick = {0,171,256},
								label style={font=\footnotesize},
								x tick label style={font=\strut},
								symbolic y coords = {\tool, \sys, \flie },
								legend columns = 2,
								legend style={at={(0.5,1.3)},
									anchor=north,draw=none,fill=none},
								ytick=data,
								ytick style={draw=none},
								xmin=0, xmax=256]
								\addplot+[green!50] plot coordinates {(171,\flie) (220,\sys) (243,\tool)}; 
								\addplot+[red!70] plot coordinates {(85,\flie) (36,\sys) (13,\tool)}; 
								
								\legend{\strut Successful, \strut Timed out}
							\end{axis} 
						\end{tikzpicture}
					\end{subfigure}
			}}
		\end{center}
		\centering
		\subcaptionbox{Size comparison\label{subfig:size-comp}}{
			\begin{subfigure}[b]{0.25\textwidth}
				\centering
				\begin{tikzpicture}
					\begin{axis}[
						height=45mm,
						width=45mm,
						xmin=1, xmax=8,
						ymin=1, ymax=8,
						enlarge x limits=true, enlarge y limits=true,
						xlabel = {\flie{} size},
						ylabel = {\tool{} size},
						ylabel near ticks,
						xtick={2,4,6,8},
						ytick={2,4,6,8},
						x tick label style={font={\strut\tiny}},
						y tick label style={font={\strut\tiny}},
						label style={font=\footnotesize},
						x label style = {yshift=1mm}
						]
						\addplot[scatter=true,
						only marks,
						mark options={fill=gray},
						visualization depends on = {3*\thisrow{Log-FvC} \as \perpointmarksize},
						scatter/@pre marker code/.style={/tikz/mark size=\perpointmarksize},
						scatter/@post marker code/.style={}] table [col sep=comma,x={Flie-size},y={Collab-size-1},meta index=2] {data/RQ1-size-comparison.csv};
						every mark/.append style={solid, fill=gray},
						\draw[black!25] (rel axis cs:0, 0) -- (rel axis cs:1, 1);
					\end{axis}
				\end{tikzpicture}
			\end{subfigure}
			\hskip 2mm
			\begin{subfigure}[b]{0.25\textwidth}
				\centering
				\begin{tikzpicture}
					\begin{axis}[
						height=45mm,
						width=45mm,
						xmin=1, xmax=8,
						ymin=1, ymax=8,
						enlarge x limits=true, enlarge y limits=true,
						xlabel = {\sys{} size},
						ylabel = {},
						xtick={2,4,6,8},
						ytick={2,4,6,8},
						label style={font=\footnotesize},
						x tick label style={font=\strut\tiny},
						x label style = {yshift=1mm},
						yticklabels = {}
						]
						\addplot[scatter=true,
						only marks,
						mark options={fill=gray},
						visualization depends on = {3*\thisrow{Log-SvC} \as \perpointmarksize},
						scatter/@pre marker code/.style={/tikz/mark size=\perpointmarksize},
						scatter/@post marker code/.style={}] table [col sep=comma,x={Sys-size},y={Collab-size-2},meta index=2] {data/RQ1-size-comparison.csv};
						every mark/.append style={solid, fill=gray},
						\draw[black!25] (rel axis cs:0, 0) -- (rel axis cs:1, 1);
					\end{axis}
				\end{tikzpicture}
			\end{subfigure}
		}
		
		\caption{Comparison of \tool, \flie{} and \sys{} on synthetic benchmarks. In Figure~\ref{subfig:runtime-comp}, all times are in seconds and `TO' indicates timeouts. The size of bubbles in the figure indicate the number of samples for each datapoint.}
		\vspace{-0.5cm}
		\label{fig:RQ1-comp}
	\end{figure*}
	
	\subsection{RQ1: Performance Comparison}
	
	To address our first research question, we have compared all three tools on a synthetic benchmark suite generated from eight $\LTLf$ formulas.
	These formulas originate from a study by Dwyer et al.~\cite{DwyerAC99}, who have collected a comprehensive set of $\LTLf$ formulas arising in real-world applications (see Table~\ref{tab:LTL-patterns} for an excerpt).
	The selected $\LTLf$ formulas have, in fact, also been used by \flie{} for generating its benchmarks.
	While \flie{} also considered formulas with $\lU$-operator, we did not consider them for generating our benchmarks due to reasons mentioned in the experimental setup.

	Our benchmark suite consists of a total of 256 samples (32 for each of the eight $\LTLf$ formulas) generated using our generation method.
	The number of traces in the samples ranges from 50 to 2\,000, while the length of traces ranges from 8 to 15.
	
	Figure~\ref{subfig:runtime-comp} presents the runtime comparison of \flie{}, \sys{} and \tool{} on all 256 samples.
	From the scatter plots, we observe that \tool{} ran faster than \flie{} on all samples. Likewise,  \tool{} was faster than \sys\ on all but eight (out of 256) samples.
	\tool{} timed out on only 13 samples, while \flie{} and \sys{} timed out on 85 and 36, respectively (see Figure~\ref{subfig:timeouts}).
	
	The good performance of \tool{} can be attributed to its efficient formula search technique.
	In particular, \tool{} only considers formulas that have a high potential of being a separating formula since it extracts Directed $\LTLf$  formulas from the sample itself.
	\flie{} and \sys{}, on the other hand, search through arbitrary formulas (in order of increasing size), each time checking if the current one separates the sample.
	
	Figure~\ref{subfig:size-comp} presents the comparison of the size of the formulas inferred by each tool.
	On 170 out of the 256 samples, all tools terminated and returned an $\LTLf$ formula with size at most 7.
	In 150 out of this 170 samples, \tool{}, \flie{}, and \sys{} inferred formulas of equal size, while on the remaining 20 samples \tool{} inferred formulas that were larger.
	The latter observation indicates that \tool{} misses certain small, separating formulas, in particular, the ones which are not a Boolean combination of directed formulas.
	
	However, it is important to highlight that the formulas learned by \tool{} are in most cases not significantly larger than those learned by \flie{} and \sys{}. 
	This can be seen from the fact that the average size of formulas inferred by \tool{} (on benchmarks in which none of the tools timed out) is 3.21, while the average size of formulas inferred by \flie{} and \sys{} is 3.07.
	
	To ensure that \tool{} performs well, not only in our generated benchmarks, we compared the performance of the tools on an existing benchmark suite\footnote{https://github.com/cryhot/samples2LTL}~\cite{abs-2104-15083}. 
	The benchmark suite has been generated using a fairly naive generation method from the same set of $\LTLf$ formulas listed in Table~\ref{tab:LTL-patterns}.
	
	Figure~\ref{subfig:time-comp-flie-benchmarks} represents the runtime comparison of \flie{}, \sys{} and \tool{} on 98 samples. From the scatter plots, we observe that \tool{} runs much faster than \flie{} on all samples and than \sys{} on all but two samples. Also, \tool{} timed out only on 3 samples while \sys{} timed out on 6 samples and \flie{} timed out on 15 samples.  
	
	Figure~\ref{subfig:size-comp-flie-benchmarks} presents the comparison of formula size inferred by each tool. On 84 out of 98 samples, where none of the tools timed out, we observe that on 65 samples, \tool{} inferred formula size equal to the one inferred by \sys{} and \flie{}.
	Further, in the samples where \tool{} learns larger formulas than other tools, the size gap is not significant. 
	This is evident from the fact that the average formula size learned by \tool{} is 4.13 which is slightly higher than that by \flie{} and \sys{}, 3.84.
	
	\begin{figure*}[t]
		\centering
		\begin{center}
			\subcaptionbox{Runtime comparison\label{subfig:time-comp-flie-benchmarks}}{
				\begin{subfigure}[b]{0.25\textwidth}
					\centering
					\begin{tikzpicture}
						\begin{loglogaxis}[
							height=45mm,
							width=45mm,
							xmin=1e-1, ymin=1e-1,
							enlarge x limits=true, enlarge y limits=true,
							xlabel = {\flie{} time},
							ylabel = {\tool{} time},
							xtickten={-1, 0, 1, 2},%
							ytickten={-1, 0, 1, 2},%
							extra x ticks={2000}, extra x tick labels={\strut TO},%
							extra y ticks={2000}, 
							extra y tick labels={\strut TO},
							xminorticks=false,
							yminorticks=false,
							xlabel near ticks,
							ylabel near ticks,
							label style={font=\footnotesize},
							x tick label style={font={\strut\tiny}},
							y tick label style={font={\strut\tiny}},		
							x label style = {yshift=1mm},
							]
							
							\addplot[
							scatter=false,
							only marks,
							mark=x,
							mark size=1.5,
							]
							table [col sep=comma,x={ Flie Time},y={Collab Time}] {data/RQ1-flie-tool-comparison.csv};
							
							\draw[black!25] (rel axis cs:0, 0) -- (rel axis cs:1, 1);
							
						\end{loglogaxis}
					\end{tikzpicture}
				\end{subfigure}
				\hskip 15mm
				\begin{subfigure}[b]{0.25\textwidth}
					\centering
					\begin{tikzpicture}
						\begin{loglogaxis}[
							height=45mm,
							width=45mm,
							xmin=1e-1, ymin=1e-1,
							enlarge x limits=true, enlarge y limits=true,
							xlabel = {\sys{} time},
							ylabel = {},
							xtickten={-1, 0, 1, 2},%
							ytickten={-1, 0, 1, 2},
							extra x ticks={2000}, extra x tick labels={\strut TO},%
							extra y ticks={2000}, 
							extra y tick labels={},
							xminorticks=false,
							yminorticks=false,
							xlabel near ticks,
							ylabel near ticks,
							label style={font=\footnotesize},
							x tick label style={font={\strut\tiny}},		
							x label style = {yshift=1mm},
							yticklabels = {},
							]
							
							\addplot[
							scatter=false,
							only marks,
							mark=x,
							mark size=1.5,
							]
							table [col sep=comma,x={Syslite Time},y={Collab Time}] {data/RQ1-flie-tool-comparison.csv};
							
							\draw[black!25] (rel axis cs:0, 0) -- (rel axis cs:1, 1);
							
						\end{loglogaxis}
					\end{tikzpicture}
			\end{subfigure}}
		\end{center}
		
		\begin{center}
			\subcaptionbox{Size comparison\label{subfig:size-comp-flie-benchmarks}}{
				\begin{subfigure}[b]{0.25\textwidth}
					\centering
					\begin{tikzpicture}
						\begin{axis}[
							height=45mm,
							width=45mm,
							xmin=1, xmax=8,
							ymin=1, ymax=8,
							enlarge x limits=true, enlarge y limits=true,
							xlabel = {\flie{} size},
							ylabel = {\tool{} size},
							ylabel near ticks,
							xtick={2,4,6,8},
							ytick={2,4,6,8},
							x tick label style={font={\strut\tiny}},
							y tick label style={font={\strut\tiny}},
							label style={font=\footnotesize},
							x label style = {yshift=1mm}
							]
							\addplot[scatter=true,
							only marks,
							mark options={fill=gray},
							visualization depends on = {2*\thisrow{Log-FvC} \as \perpointmarksize},
							scatter/@pre marker code/.style={/tikz/mark size=\perpointmarksize},
							scatter/@post marker code/.style={}] table [col sep=comma,x={Flie-size},y={Collab-size-1},meta index=2] {data/RQ1-flie-size-comparison.csv};
							every mark/.append style={solid, fill=gray},
							\draw[black!25] (rel axis cs:0, 0) -- (rel axis cs:1, 1);
						\end{axis}
					\end{tikzpicture}
				\end{subfigure}
				\hskip 8mm
				\begin{subfigure}[b]{0.25\textwidth}
					\centering
					\begin{tikzpicture}
						\begin{axis}[
							height=45mm,
							width=45mm,
							xmin=1, xmax=8,
							ymin=1, ymax=8,
							enlarge x limits=true, enlarge y limits=true,
							xlabel = {\sys{} size},
							ylabel = {},
							xtick={2,4,6,8},
							ytick={2,4,6,8},
							label style={font=\footnotesize},
							x tick label style={font=\strut\tiny},
							x label style = {yshift=1mm},
							yticklabels = {}
							]
							\addplot[scatter=true,
							only marks,
							mark options={fill=gray},
							visualization depends on = {2*\thisrow{Log-SvC} \as \perpointmarksize},
							scatter/@pre marker code/.style={/tikz/mark size=\perpointmarksize},
							scatter/@post marker code/.style={}] table [col sep=comma,x={Syslite-size},y={Collab-size-2},meta index=2] {data/RQ1-flie-size-comparison.csv};
							every mark/.append style={solid, fill=gray},
							\draw[black!25] (rel axis cs:0, 0) -- (rel axis cs:1, 1);
						\end{axis}
					\end{tikzpicture}
			\end{subfigure}}
		\end{center}
		\caption{Comparison of \tool, \flie{} and \sys{} on existing benchmarks. In Figure~\ref{subfig:time-comp-flie-benchmarks}, all times are in seconds and `TO' indicates timeouts. The size of bubbles indicate the number of samples for each datapoint.}
		\label{fig:size-comp-app}
	\end{figure*}

	Overall, \tool{} displayed significant speed-up over both \flie{} and \sys{} while learning a formula similar in size, answering question RQ1 in the positive.

	\begin{figure}[!htbp]
		\vspace{-0.2cm}
		\begin{center}
			\subcaptionbox{Scalability in sample size\label{subfig:sample-scale}}
			{
				\hspace{1.3cm}
				\begin{subfigure}[b]{0.4\textwidth}
					\begin{tikzpicture}
						\begin{loglogaxis}[
							height=40mm,
							xmin=1e2, ymin=1e0,
							enlarge x limits=true, enlarge y limits=true,
							xlabel = {Number of traces},
							ylabel = {Average Time},
							xtickten={2,3,4,5},%
							ytickten= {0, 1, 2},%
							extra y ticks={900}, 
							extra y tick labels={\strut TO},
							xminorticks=false,
							yminorticks=false,
							xlabel near ticks,
							ylabel near ticks,
							label style={font=\footnotesize},
							x tick label style={font={\strut\tiny}},
							y tick label style={font={\strut\tiny}},		
							x label style = {yshift=1mm},
							legend columns = 3,
							legend style={at={(1.4,1.6)},
								anchor=north,fill=none},
							title = {Formula $\varphi_\mathit{cov}$},
							title style = {font=\footnotesize, inner sep=-4pt}
							]
							\addplot[mark=square,blue] plot coordinates {
								(200,4.4)
								(1000,21)
								(6000,125.26)
								(10000,210.48)
								(50000,771.04)
								(100000,900)
							};
							\addlegendentry{\strut\tool{}}
							
							\addplot[color=red,mark=x,opacity =0.8]
							plot coordinates {
								(200,423.3)
								(1000,900)
								(6000,900)
								(10000,900)
								(50000,900)
								(100000,900)
							};
							\addlegendentry{\strut\sys{}}

							\addplot[mark=*,green,opacity =0.7] plot coordinates {
								(200,900)
								(1000,900)
								(6000,900)
								(10000,900)
								(50000,900)
								(100000,900)
							};
							\addlegendentry{\strut\flie{}}
							
						\end{loglogaxis}
					\end{tikzpicture}
				\end{subfigure}
				\hskip 1mm
				
				\begin{subfigure}[b]{0.4\textwidth}
					\begin{tikzpicture}
						\begin{loglogaxis}[
							height=40mm,
							xmin=1e2, ymin=1e1,
							xmax=1e4,
							enlarge x limits=true, enlarge y limits=true,
							xlabel = {Number of traces},
							xtickten={2,3,4},%
							ytickten= {1, 2},%
							extra y ticks={900}, 
							extra y tick labels={\strut TO},
							xminorticks=false,
							yminorticks=false,
							xlabel near ticks,
							ylabel near ticks,
							label style={font=\footnotesize},
							x tick label style={font={\strut\tiny}},
							y tick label style={font={\strut\tiny}},		
							x label style = {yshift=1mm},
							legend columns = 0,
							legend style={draw=none},
							title = {Formula $\varphi_\mathit{seq}$},
							title style = {font=\footnotesize, inner sep=-4pt}
							]
							\addplot[mark=square,blue] plot coordinates {
								(200,42.8)
								(500,91.9)
								(1000,177.4)
								(2000,332.2)
								(4000,673.29)
								(6000,900)
								(10000,900)
							};
							\addlegendentry{\strut\tiny\tool{}}
							
							\addplot[color=red,mark=x,opacity =0.8]
							plot coordinates {
								(200,114.3)
								(500,265)
								(1000,509)
								(2000,900)
								(4000,900)
								(6000,900)
								(10000,900)
								
							};
							\addlegendentry{\strut\tiny\sys{}}

							\addplot[mark=*,green,opacity =0.7] plot coordinates {
								(200,900)
								(500,900)
								(1000,900)
								(2000,900)
								(4000,900)
								(6000,900)
								(10000,900)
							};
							\addlegendentry{\strut\tiny\flie{}}
							
							\legend{}
						\end{loglogaxis}
					\end{tikzpicture}
			\end{subfigure}}
		\end{center}
		\begin{center}
			\subcaptionbox{Scalability in sample lengths\label{subfig:length-scale}}
			{	\hspace{-10mm}
				\begin{subfigure}[b]{0.25\textwidth}
					\centering
					\begin{tikzpicture}
						\begin{semilogyaxis}[
							height=40mm,
							xmin=10, ymin=1e0,
							xmax=50,
							enlarge x limits=true, enlarge y limits=true,
							xlabel = {Length of traces},
							ylabel = {Average Time},
							xtickten={2,3,4,5},%
							ytickten= {0, 1, 2},%
							extra y ticks={900}, 
							extra y tick labels={\strut TO},
							xminorticks=false,
							yminorticks=false,
							xlabel near ticks,
							ylabel near ticks,
							label style={font=\footnotesize},
							x tick label style={font={\strut\tiny}},
							y tick label style={font={\strut\tiny}},		
							x label style = {yshift=1mm},
							legend columns = 3,
							legend style={draw=none,
								anchor=north,fill=none},
							title = {Formula $\varphi_\mathit{cov}$},
							title style = {font=\footnotesize, inner sep=-4pt}
							]
							\addplot[mark=square,blue] plot coordinates {
								(10,3.2)
								(20,5.35)
								(30,7.64)
								(40,9.94)
								(50,12.04)
							};
							\addlegendentry{\strut\tool{}}
							
							\addplot[color=red,mark=x,opacity =0.8]
							plot coordinates {
								(10,428.86)
								(20,717.35)
								(30,307.48)
								(40,380.73)
								(50,468.50)
							};
							\addlegendentry{\strut\sys{}}

							\addplot[mark=*,green,opacity =0.7] plot coordinates {
								(10,900)
								(20,900)
								(30,900)
								(40,900)
								(50,900)
							};
							\addlegendentry{\strut\flie{}}
							
							\legend{}
						\end{semilogyaxis}
					\end{tikzpicture}
				\end{subfigure}
				\hskip 22mm
				
				\begin{subfigure}[b]{0.25\textwidth}
					\centering
					\begin{tikzpicture}
						\begin{semilogyaxis}[
							height=40mm,
							xmin=10, ymin=1e1,
							xmax=50,
							enlarge x limits=true, enlarge y limits=true,
							xlabel = {Length of traces},
							xtickten={2,3,4},%
							ytickten= {1, 2},%
							extra y ticks={900}, 
							extra y tick labels={\strut TO},
							xminorticks=false,
							yminorticks=false,
							xlabel near ticks,
							ylabel near ticks,
							label style={font=\footnotesize},
							x tick label style={font={\strut\tiny}},
							y tick label style={font={\strut\tiny}},		
							x label style = {yshift=1mm},
							legend columns = 0,
							legend style={draw=none},
							title = {Formula $\varphi_\mathit{seq}$},
							title style = {font=\footnotesize, inner sep=-4pt}
							]
							\addplot[mark=square,blue] plot coordinates {
								(10,41.39)
								(20,198.81)
								(30,504.20)
								(40,784.69)
								(50,900)
							};
							\addlegendentry{\strut\tiny\tool{}}
							
							\addplot[color=red,mark=x,opacity =0.8]
							plot coordinates {
								(10,113.87)
								(20,205.28)
								(30,303.49)
								(40,391.68)
								(50,484.82)
								
							};
							\addlegendentry{\strut\tiny\sys{}}

							\addplot[mark=*,green,opacity =0.7] plot coordinates {
								(10,900)
								(20,900)
								(30,900)
								(40,900)
								(50,900)
							};
							\addlegendentry{\strut\tiny\flie{}}
							
							\legend{}
						\end{semilogyaxis}
					\end{tikzpicture}
			\end{subfigure}}
		\end{center}
		\begin{center}
			\subcaptionbox{Scalability in formula size\label{subfig:formula-scale}}{
				\hspace{1.3cm}
				\begin{subfigure}[b]{0.4\textwidth}
					\begin{tikzpicture}
						\begin{semilogyaxis}[
							height=40mm,
							xmin=2, ymin=1e-1,
							enlarge x limits=true, enlarge y limits=true,
							xlabel = {Formula size parameter (n)},
							ylabel = {Average Time},
							xtick={2,3,4,5},%
							ytickten= {-1, 0, 1, 2},%
							extra y ticks={900}, 
							extra y tick labels={\strut TO},
							xminorticks=false,
							yminorticks=false,
							xlabel near ticks,
							ylabel near ticks,
							label style={font=\footnotesize},
							x tick label style={font={\strut\tiny}},
							y tick label style={font={\strut\tiny}},		
							x label style = {yshift=1mm},
							legend columns = 3,
							title = {Formula family $\varphi^n_\mathit{cov}$},
							title style = {font=\footnotesize, inner sep=-4pt}
							]
							\addplot[mark=square,blue] plot coordinates {
								(2,0.25)
								(3,3.54)
								(4,76.07)
								(5,900)
							};
							\addlegendentry{\strut\tool{}}
							
							\addplot[color=red,mark=x,opacity =0.8]
							plot coordinates {
								(2,2.57)
								(3,425.17)
								(4,900)
								(5,900)
							};
							\addlegendentry{\strut\sys{}}

							\addplot[mark=*,green,opacity =0.7] plot coordinates {
								(2,118.37)
								(3,900)
								(4,900)
								(5,900)
							};
							\addlegendentry{\strut\flie{}}
							\legend{}
							
						\end{semilogyaxis}
					\end{tikzpicture}
				\end{subfigure}
				\hskip 1mm
				\begin{subfigure}[b]{0.4\textwidth}
					\begin{tikzpicture}
						\begin{semilogyaxis}[
							height=40mm,
							xmin=2, ymin=1e-1,
							enlarge x limits=true, enlarge y limits=true,
							xlabel = {Formula size parameter (n)},
							xtick={2,3,4,5},%
							ytickten= {-1, 0, 1, 2},%
							extra y ticks={900}, 
							extra y tick labels={\strut TO},
							xminorticks=false,
							yminorticks=false,
							xlabel near ticks,
							ylabel near ticks,
							label style={font=\footnotesize},
							x tick label style={font={\strut\tiny}},
							y tick label style={font={\strut\tiny}},		
							x label style = {yshift=1mm},
							legend columns = 0,
							legend style={draw=none},
							title = {Formula family $\varphi^n_\mathit{seq}$},
							title style = {font=\footnotesize, inner sep=-4pt}
							]
							\addplot[mark=square,blue] plot coordinates {
								(2,0.52)
								(3,40)
								(4,900)
								(5,900)
							};
							\addlegendentry{\strut\tool{}}
							
							\addplot[color=red,mark=x,opacity =0.8]
							plot coordinates {
								(2,0.89)
								(3,114.42)
								(4,900)
								(5,900)
							};
							\addlegendentry{\strut\sys{}}

							\addplot[mark=*,green,opacity =0.7] plot coordinates {
								(2,143.46)
								(3,900)
								(4,900)
								(5,900)
							};
							\addlegendentry{\strut\tiny\flie{}}
							
							\legend{}
						\end{semilogyaxis}
					\end{tikzpicture}
					
			\end{subfigure}}
		\end{center}
		
		\vspace{-0.4cm}
		\caption{Comparison of \tool, \flie{} and \sys{} on synthetic benchmarks. In Figures~\ref{subfig:sample-scale} and~\ref{subfig:length-scale}, all times are in seconds and `TO' indicates timeouts.}
		\label{fig:scale-results}
		
		\vspace{-0.2cm}
		
	\end{figure}
	\subsection{RQ2: Scalability}
	To address the second research question, we investigate the scalability of \tool{} in two dimensions: the size of the sample and the size of the formula from which the samples are generated.
	
	\paragraph*{Scalability with respect to the size of the samples.} For demonstrating the scalability with respect to the size of the samples, we consider two formulas $\varphi_\mathit{cov}=\lF(a_1) \land \lF (a_2) \land \lF (a_3)$ and $\varphi_\mathit{seq}=\lF(a_1 \land \lF (a_2 \land \lF a_3))$, both of which appear commonly in robotic motion planning~\cite{FainekosKP05}.
	While the formula $\varphi_\mathit{cov}$ describes the property that a robot eventually visits (or covers) three regions $a_1$, $a_2$, and $a_3$ in arbitrary order, the formula $\varphi_\mathit{seq}$ describes that the robot has to visit the regions in the specific order $a_1a_2a_3$.
	
	We have generated two sets of benchmarks for both formulas for which we varied the number of traces and their length, respectively.
	More precisely, the first benchmark set contains 90 samples of an increasing number of traces (5 samples for each number), ranging from 200 to 100\,000, each consisting of traces of fixed length 10.
	On the other hand, the second benchmark set contains 90 samples of 200 traces, containing traces from length 10 to length 50.
	
	
	Figure~\ref{subfig:sample-scale} shows the average runtime results of \tool{}, \flie{}, and \sys{} on the first benchmark set.
	We observe that \tool{} substantially outperformed the other two tools on all samples.
	This is because both $\varphi_\mathit{cov}$ and $\varphi_\mathit{seq}$ are of size eight and inferring formulas of such size is computationally challenging for \flie{} and \sys{}.
	In particular, \flie{} and \sys{} need to search through all formulas of size up to eight to infer the formulas, while, \tool, due to its efficient search order (using length and width of a formula), infers them faster.
	
	From Figure~\ref{subfig:sample-scale}, we further observe a significant difference between the run times of \tool{} on samples generated from formula $\varphi_\mathit{cov}$ and from formula $\varphi_\mathit{seq}$. 
	This is evident from the fact that \tool{} failed to infer formulas for samples of $\varphi_\mathit{seq}$ starting at a size of 6\,000, while it could infer formulas for samples of $\varphi_\mathit{cov}$ up to a size of 50\,000.
	Such a result is again due to the search order used by \tool{}:
	while $\varphi_\mathit{cov}$ is a Boolean combination of directed formulas of length 1 and width 1, $\varphi_\mathit{seq}$ is a directed formula of length~3 and width~1.

	Figure~\ref{subfig:length-scale} depicts the results we obtained by running all the second benchmark set with varying trace lengths.
	Some trends we observe here are similar to the ones we observe in the first benchmark set.
	For instance, \tool{} performs better on the samples from $\varphi_\mathit{cov}$ than it does on samples from $\varphi_\mathit{seq}$.
	The reason for this remains similar: it is easier to find a formula which is a boolean combination of length 1, width 1 $\dLTL$, than a simple $\LTLf$ of length 3 and width 1.
	
	Contrary to the results on the first benchmark set, we observe that the increase of runtime with the length of the sample is quadratic.
	This explains why on samples from $\varphi_\mathit{seq}$ on large lengths such as 50, \tool{} faces time-out.
	However, for samples from $\varphi_\mathit{cov}$, \tool{} displays the ability to scale way beyond length 50.

	\paragraph*{Scalability with respect to the size of the formula.} To demonstrate the scalability with respect to the size of the formula used to generate samples, we have extended $\varphi_\mathit{cov}$ and $\varphi_\mathit{seq}$ to families of formulas $(\varphi^n_\mathit{cov})_{n \in \mathbb N \setminus \{ 0 \}}$ with $\varphi^n_\mathit{cov} = \lF(a_1) \land \lF (a_2) \land \ldots \land \lF(a_n)$ and $(\varphi^n_\mathit{seq})_{n \in \mathbb N \setminus \{ 0 \}}$ with $\varphi^n_\mathit{seq} = \lF(a_1 \land \lF (a_2 \land \lF(\ldots \land \lF a_n)))$, respectively. 
	These family of formulas describe properties similar to that of $\varphi_\mathit{cov}$ and $\varphi_\mathit{seq}$, but the number of regions is parameterized by $n \in \mathbb N \setminus \{ 0 \}$.
	We consider formulas from the two families by varying $n$ from 2 to 5 to generate a benchmark suite consisting of samples (5 samples for each formula) having 200 traces of length 10. 
	
	Figure~\ref{subfig:formula-scale} shows the average run time comparison of the tools for samples from increasing formula sizes.
	We observe a trend similar to Figure~\ref{subfig:sample-scale}: \tool{} performs better than the other two tools and infers formulas of family $\varphi^n_\mathit{cov}$ faster than that of $\varphi^n_\mathit{seq}$. 
	However, contrary to the near linear increase of the runtime with the number of traces, we notice an almost exponential increase of the runtime with the formula size.
	
	Overall, our experiments show better scalability with respect to sample and formula size compared against the other tools, answering RQ2 in the positive.

	\subsection{RQ3: Anytime Property}
	To answer RQ3, we list two advantages of the anytime property of our algorithm.
	We demonstrate these advantages by showing evidence from the runs of \tool{} on benchmarks used in RQ1 and RQ2.
	
	First, in the instance of a time out, our algorithm may find a ``concise'' separating formula while the other tools will not.
	In our experiments, we observed that for all benchmarks used in RQ1 and RQ2, \tool{} obtained a formula even when it timed out.
	In fact, in the samples from $\varphi^5_\mathit{cov}$ used in RQ2, \tool{} (see Figure~\ref{subfig:formula-scale}) obtained the exact original formula, that too within one second (0.7 seconds in average), although timed out later.
	The time out was because \tool{} continued to search for smaller formulas even after finding the original formula.
	
	Second, our algorithm can actually output the final formula earlier than its termination.
	This is evident from the fact that, for the 243 samples in RQ1 where \tool{} does not time out, the average time required to find the final formula is 10.8 seconds, while the average termination time is 25.17 seconds.
	Thus, there is a chance that even if one stops the algorithm earlier than its termination, one can still obtain the final formula.
	
	Our observations from the experiments clearly indicate the advantages of anytime property to obtain a concise separating formula and thus, answering RQ3 in the positive.

\subsection{RQ4: Noisy Setting}
To answer RQ4, we compared the noisy learning setting of $\tool$ against a state-of-the-art tool~\cite{abs-2104-15083} that uses MaxSAT based approach to infer $\LTLf$  from noisy data. For this comparison, we ran both tools on the same synthetic benchmark suite introduced to answer RQ1 consisting of 256 samples. We tested for three noise thresholds ($1\%, 5\%$, and $10\%$) and presented the runtime comparison in Figure~\ref{subfig:noisyruntime-comp}. This comparison clearly indicates the efficiency and better performance of our algorithm against the MaxSAT-based approach. Figure~\ref{subfig:noisy-cactus} represents a cactus plot that shows the cumulative number of samples from which both tools inferred $\LTLf$  formulas over time for the same three noise levels. This clearly validates the scalability of our algorithm in the noisy setting.

\begin{figure}[H]
	\centering
	\subcaptionbox*{ }{%
		\centering
		\hskip 6mm
		\begin{tikzpicture}
			\node[draw, rectangle, minimum width=2mm, minimum height=2mm, fill=black] (black) {};
			\node[right=1mm of black, anchor=west](blacked) {noise $\leq$ 1\%};
			\node[draw, rectangle, minimum width=2mm, minimum height=2mm, fill=red, right=6mm of blacked] (red) {};
			\node[right=1mm of red, anchor=west](reded) {noise $\leq$ 5\%};
			\node[draw, rectangle, minimum width=2mm, minimum height=2mm, fill=blue, right=6mm of reded] (blue) {};
			\node[right=1mm of blue, anchor=west] {noise $\leq$ 10\%};
		\end{tikzpicture}
	}
	\subcaptionbox{Runtime comparison\label{subfig:noisyruntime-comp}}{%
		\begin{tikzpicture}
			\begin{loglogaxis}[
				height=45mm,
				width=45mm,
				xmin=1e-1, ymin=1e-1,
				enlarge x limits=true, enlarge y limits=true,
				xlabel = {MaxSAT time},
				ylabel = {Scarlet time},
				xtickten={-1, 0, 1, 2},
				ytickten={-1, 0, 1, 2},
				extra x ticks={1000}, extra x tick labels={\strut TO},
				extra y ticks={1000}, extra y tick labels={\strut TO},
				xminorticks=false,
				yminorticks=false,
				xlabel near ticks,
				ylabel near ticks,
				label style={font=\footnotesize},
				x tick label style={font={\strut\tiny}},
				y tick label style={font={\strut\tiny}},
				x label style = {yshift=1mm},
				]
				\addplot[
				scatter=false,
				only marks,
				mark=x,
				mark size=1.5,
				]
				table [col sep=comma,x={max_sat Time},y={Scarlet Time}] {data/noisy-comparison-0.01.csv};
				\draw[black!25] (rel axis cs:0, 0) -- (rel axis cs:1, 1);
			\end{loglogaxis}
		\end{tikzpicture}
		\hskip 6mm
		\begin{tikzpicture}
			\begin{loglogaxis}[
				height=45mm,
				width=45mm,
				xmin=1e-1, ymin=1e-1,
				enlarge x limits=true, enlarge y limits=true,
				xlabel = {MaxSAT time},
				xtickten={-1, 0, 1, 2},
				extra x ticks={1000}, extra x tick labels={\strut TO},
				ytickten={-1, 0, 1, 2},
				extra y tick labels={},
				xminorticks=false,
				yminorticks=false,
				xlabel near ticks,
				ylabel near ticks,
				label style={font=\footnotesize},
				legend columns = 3,
				legend style={at={(1.4,1.6)},
					anchor=north,fill=none},
				x tick label style={font={\strut\tiny}},		
				x label style = {yshift=1mm},
				yticklabels = {},
				]
				\addplot[
				scatter=false,
				only marks,
				mark=x,
				mark size=1,
				red, 
				]
				table [col sep=comma,x={max_sat Time},y={Scarlet Time}] {data/noisy-comparison-0.05.csv};
				\draw[black!25] (rel axis cs:0, 0) -- (rel axis cs:1, 1);
				\addlegendentry{\strut\tiny\scarlet{}}
				\legend{}
			\end{loglogaxis}
		\end{tikzpicture}
		\hskip 6mm
		\begin{tikzpicture}
			\begin{loglogaxis}[
				height=45mm,
				width=45mm,
				xmin=1e-1, ymin=1e-1,
				enlarge x limits=true, enlarge y limits=true,
				xlabel = {MaxSAT time},
				xtickten={-1, 0, 1, 2},
				ytickten={-1, 0, 1, 2},
				extra y tick labels={},
				extra x ticks={1000}, extra x tick labels={\strut TO},
				xminorticks=false,
				yminorticks=false,
				xlabel near ticks,
				ylabel near ticks,
				legend columns = 3,
				legend style={at={(1.4,1.6)},
					anchor=north,fill=none},
				label style={font=\footnotesize},
				x tick label style={font={\strut\tiny}},		
				x label style = {yshift=1mm},
				yticklabels = {},
				]
				\addplot[
				scatter=false,
				only marks,
				mark=x,
				mark size=1, blue
				]
				table [col sep=comma,x={max_sat Time},y={Scarlet Time}] {data/noisy-comparison-0.1.csv};
				\draw[black!25] (rel axis cs:0, 0) -- (rel axis cs:1, 1);
			\end{loglogaxis}
		\end{tikzpicture}
	}
	\subcaptionbox{Cactus plot\label{subfig:noisy-cactus}}{%
		\includegraphics[width=0.5\textwidth]{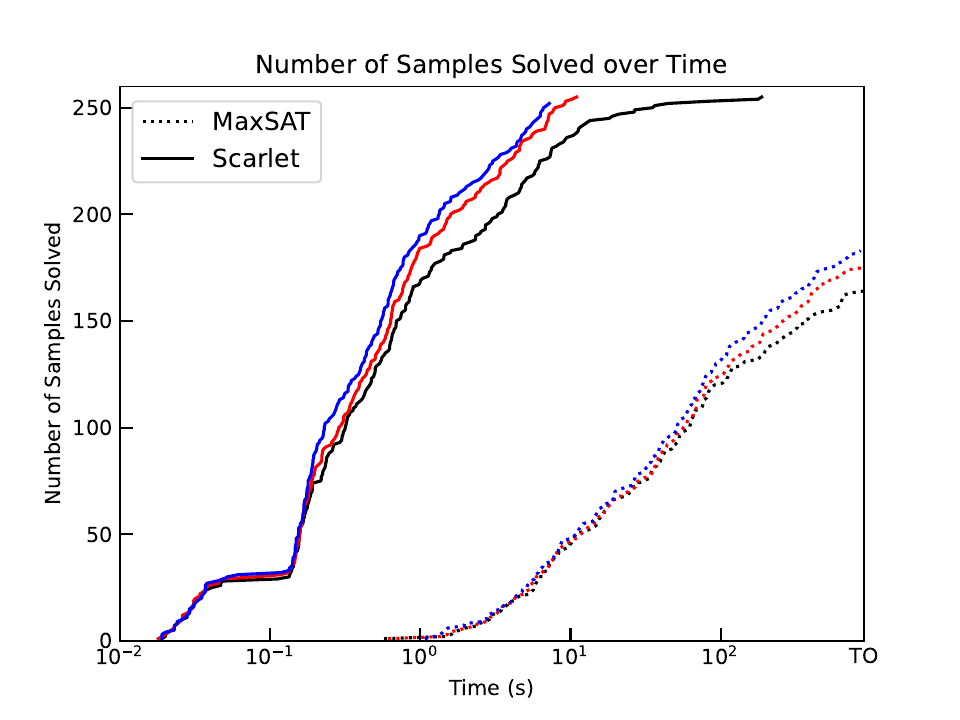}
	}
	\caption{Comparison of \tool{} and MaxSAT on synthetic noisy benchmarks for different noise thresholds. All times are in seconds, and `TO' indicates timeouts. The scatter plots in Figure~\ref{subfig:noisyruntime-comp} represent the comparison of time taken to solve each sample by both tools for three different noise levels. The cactus plot in Figure~\ref{subfig:noisy-cactus} represents the total number of samples solved within a given time-point by both the tools, and steeper lines represent better performance.}
	\label{fig:RQ4-comp}
\end{figure}
\vspace{-0.2cm}

\section{Conclusion}
We have proposed a new approach for learning temporal properties from examples, fleshing it out in an approximation anytime algorithm. We have shown in experiments that our algorithm outperforms existing tools in two ways: 
it scales to larger formulas and input samples, 
and even when it timeouts it often outputs a separating formula.

Our algorithm targets a strict fragment of $\LTLf$, restricting its expressivity in two aspects:
it does not include the $\lU$ (\textit{Until}) operator, and we cannot nest the eventually and globally operators.
We leave for future work to extend our algorithm to full $\LTL$.

An important open question concerns the theoretical guarantees offered by the greedy algorithm for the Boolean set cover problem. It extends a well known algorithm for the classic set cover problem and this restriction has been proved to yield an optimal $\log(n)$-approximation. Do we have similar guarantees in our more general setting? 

\bibliographystyle{plainnat}
\bibliography{bib}
\end{document}